\documentclass{article}

\usepackage[final]{neurips_2020}
\usepackage{natbib}

\usepackage[utf8]{inputenc} % allow utf-8 input
\usepackage[T1]{fontenc}    % use 8-bit T1 fonts
\usepackage{hyperref}       % hyperlinks
\usepackage{url}            % simple URL typesetting
\usepackage{booktabs}       % professional-quality tables
\usepackage{amsfonts}       % blackboard math symbols
\usepackage{nicefrac}       % compact symbols for 1/2, etc.
\usepackage{microtype}      % microtypography
\usepackage{url}            % simple URL typesetting
\hypersetup{colorlinks, citecolor=BlueViolet, linkcolor=blue, urlcolor=Blue}

\usepackage{mathtools}

\def\equationautorefname~#1\null{(#1)\null}

\usepackage{amsmath}
\usepackage{multirow}
\usepackage{graphicx}
\usepackage[dvipsnames]{xcolor}

\usepackage{epstopdf}
\usepackage{caption}
\usepackage{subcaption}
\usepackage{bm}
\usepackage{pgfplots}
\usepackage{footnote}
\usepackage{stfloats}
\usepackage{placeins}
\usepackage{color,soul}
\usepackage{dsfont}
\usepackage{amssymb}
\usepackage[boxruled]{algorithm2e}
\usepackage{pbox}
\usepackage{enumerate}
\usepackage{etextools}
\usepackage{amsthm}
\usepackage{tabulary}
\usepackage{dsfont}
\usepackage{hyperref}
\usepackage{marvosym}

\def\x{{\mathbf x}}

\def\v{{\mathbf v}}

\def\z{{\mathbf z}}
\def\bz{{\bar{\z}}}
\def\Z{{\mathbf Z}}

\def\I{{\mathbf I}}

\def\D{{\mathbf D}}

\def\W{{\mathbf W}}
\def\w{{\mathbf w}}

\def\alfa{{\boldsymbol \alpha}}
\def\talfa{\tilde{{\boldsymbol \alpha}}}
\def\tD{\tilde{{\D}}}

\def\betta{{\boldsymbol \beta}}
\def\tbetta{\tilde{\boldsymbol \beta}}

\def\enc[#1]{\varphi_{#1}}

\newcommand{\cmnt}[1]{\qquad \text{\textcolor{Gray}{#1}}}

\usepackage{thmtools, thm-restate}
\newtheorem{theorem}{Theorem}[section]
\newtheorem{lemma}[theorem]{Lemma}

\newtheorem{remark}[theorem]{Remark}

\newtheorem{corollary}[theorem]{Corollary}

\title{Adversarial Robustness of \\ Supervised Sparse Coding}

\author{%
  Jeremias Sulam\\
  Johns Hopkins University\\
  \texttt{jsulam1@jhu.edu}
   \And
  Ramchandran Muthukumar\\
  Johns Hopkins University \\
  \texttt{rmuthuk1@jhu.edu}
   \AND
  Raman Arora \\
   Johns Hopkins University \\
   arora@cs.jhu.edu
}

\begin{document}

\maketitle

\begin{abstract}
Several recent results provide theoretical insights into the phenomena of adversarial examples. Existing results, however, are often limited due to a gap between the simplicity of the models studied and the complexity of those deployed in practice. In this work, we strike a better balance by considering a model that involves learning a representation while at the same time giving a precise generalization bound and a robustness certificate. 
We focus on the hypothesis class obtained by combining a sparsity-promoting encoder coupled with a linear classifier, and show an interesting interplay between the expressivity and stability of the (supervised) representation map and a notion of margin in the feature space. We bound the robust risk (to $\ell_2$-bounded perturbations) of hypotheses parameterized by dictionaries that achieve a mild encoder gap on training data. Furthermore, we provide a robustness certificate for end-to-end classification. We demonstrate the applicability of our analysis by computing certified accuracy on real data, and compare with other alternatives for certified robustness.
\end{abstract}

\section{Introduction}

With machine learning applications becoming ubiquitous in modern-day life, there exists an increasing concern about the robustness of the deployed models. Since first reported in \citep{szegedy2013intriguing,goodfellow2014explaining,biggio2013evasion}, these \emph{adversarial attacks} are small perturbations of the input, imperceptible to the human eye, which can nonetheless completely fluster otherwise well-performing systems. Because of clear security implications \citep{darpa}, this phenomenon has sparked an increasing amount of work dedicated to devising defense strategies \citep{metzen2017detecting,gu2014towards,madry2017towards} and correspondingly more sophisticated attacks \citep{carlini2017adversarial,athalye2018obfuscated,tramer2020adaptive}, with each group trying to triumph over the other in an arms-race of sorts.

A different line of research attempts to understand adversarial examples from a theoretical standpoint. Some works have focused % aim to develop
on giving robustness certificates, thus providing a guarantee to withstand the attack of an adversary under certain assumptions \citep{cohen2019certified,raghunathan2018certified,wong2017provable}. Other works address questions of learnabiltiy \citep{shafahi2018adversarial,cullina2018pac,bubeck2018adversarial,tsipras2018robustness} or sample complexity \citep{schmidt2018adversarially,yin2018rademacher,tu2019theoretical}, in the hope of better characterizing the increased difficulty of learning hypotheses that are robust to adversarial attacks. 
While many of these results are promising, the analysis is often limited to simple models.

Here, we strike a better balance by considering a model that involves learning a representation while at the same time giving a precise generalization bound and a robustness certificate. 
In particular, we focus our attention on the adversarial robustness of the supervised sparse coding model \citep{mairal2011task}, or task-driven dictionary learning, consisting of a linear classifier acting on the representation computed via a supervised sparse encoder. 
We show an interesting interplay between the expressivity and stability of a (supervised) representation map and a notion of margin in the feature space. 
The idea of employing sparse representations as data-driven features for supervised learning goes back to the early days of deep learning \citep{coates2011importance,kavukcuoglu2010fast,zeiler2010deconvolutional,ranzato2007unsupervised}, and has had a significant impact on applications in computer vision and machine learning \citep{wright2010sparse,henaff2011unsupervised,mairal2008discriminative,mairal2007sparse,gu2014projective}. More recently, new connections between deep networks and sparse representations were formalized by~\cite{papyan2018theoretical}, which further helped deriving stability guarantees \citep{papyan2017working}, providing architecture search strategies and analysis \citep{tolooshams2019deep,murdock2020dataless,sulam2019multi}, and other theoretical insights \citep{xin2016maximal,aberdam2019multi,aghasi2020fast,aberdam2020ada,moreau2016understanding}. 
While some recent work has leveraged the stability properties of these latent representations to provide robustness guarantees against adversarial attacks \citep{romano2019adversarial}, these rely on rather stringent generative model assumptions that are difficult to be satisfied and verified in practice. In contrast, our assumptions rely on the existence of a positive \emph{gap} in the encoded features, as proposed originally by  \cite{mehta2013sparsity}. This distributional assumption is significantly milder -- it is directly satisfied by making traditional sparse generative model assumptions -- and can be directly quantified from data.

This work makes two main contributions: The first is a bound on the robust risk of hypotheses that achieve a mild encoder gap assumption, where the adversarial corruptions are bounded in $\ell_2$-norm. 
Our proof technique follows a standard argument based on a minimal $\epsilon$-cover of the parameter space, dating back to \cite{vapnik1971uniform} and adapted for matrix factorization and dictionary learning problems in \cite{gribonval2015sample}. However, the analysis of the Lipschitz continuity of the adversarial loss with respect to the model parameters is considerably more involved. The increase in the sample complexity is mild with adversarial corruptions of size $\nu$ manifesting as an additional term of order $\mathcal{O}\left((1+\nu)^2/m\right)$ in the bound, where $m$ is the number of samples, and a minimal encoder gap of $\mathcal{O}(\nu)$ is necessary. Much of our results extend directly to other supervised learning problems (e.g. regression).
Our second contribution is a robustness certificate that holds for every hypothesis in the function class for $\ell_2$ perturbations for multiclass classification. In a nutshell, this result guarantees that the label produced by the hypothesis will not change if the encoder gap is \emph{large enough} relative to the energy of the adversary, the classifier margin, and properties of the model (e.g. dictionary incoherence).

\section{Preliminaries and Background} \label{sec:Preliminiaries}
In this section, we first describe our notation and the learning problem, and then proceed to situate our contribution in relation to prior work. 

Consider the spaces of inputs, $\mathcal{X}\subseteq {B}_{\mathbb{R}^d}$, i.e. the unit ball in $\mathbb{R}^d$, and labels, $\mathcal{Y}$. Much of our analysis is applicable to a broad class of label spaces, but we will focus on binary and multi-class classification setting in particular. We assume that the data is sampled according to some unknown distribution $P$ over $\mathcal X \times \mathcal Y$.
Let $\mathcal{H} = \{f:\mathcal{X}\to\mathcal{Y}'\}$ denote a hypothesis class mapping inputs into some output space $\mathcal{Y}' \subseteq \mathbb{R}$. 
Of particular interest to us are norm-bounded linear predictors, $f(\cdot)=\langle \w, \cdot \rangle$, parametrized by $d$-dimensional vectors $\w\in \mathcal W = \{\w \in \mathbb{R}^d : \|\w\|_2\leq B\}$.   

From a learning perspective, we have a considerable understanding of the linear hypothesis class, both in a stochastic non-adversarial setting as well as in an adversarial context~ \citep{charles2019convergence,li2019inductive}. However, from an application standpoint, linear predictors are often too limited, and rarely applied directly on input features. Instead, most state-of-the-art systems involve learning a representation. In general, an \emph{encoder} map $\varphi:\mathcal{X}\to\mathcal{Z} \subseteq \mathbb{R}^p$, parameterized by parameters $\theta$, is composed with a linear function so that $f(\x) = \langle \w, \varphi_\theta(\x) \rangle$, for $\w\in\mathbb{R}^p$. This description applies to a large variety of popular models, including kernel-methods, multilayer perceptrons and deep convolutional neural networks.
Herein we focus on an encoder given as the solution to a Lasso problem \citep{tibshirani1996regression}. More precisely, we  consider $\varphi_\D(\x):\mathbb{R}^d\to \mathbb{R}^p$, defined by 
\begin{equation}
    \label{eqn:enc-1}
\varphi_\D(\x) \coloneqq \arg\min_\z \frac{1}{2} \|\x - \D\z\|^2_2 +\lambda \|\z\|_1.
\end{equation}
Note that when $\D$ is overcomplete, i.e. $p>d$, this problem is not strongly convex. Nonetheless, we will assume that that solution to Problem~\ref{eqn:enc-1} is unique\footnote{The solution is unique under mild assumptions \citep{tibshirani2013lasso}, and otherwise our results hold for any solution returned by a deterministic solver.}, and study the hypothesis class given by $ \mathcal{H} = \{ f_{\D,\w}(\x) = \langle \w, \varphi_\D(\x) \rangle : \w\in \mathcal{W}, \D \in \mathcal{D} \}$, where $\mathcal{W} = \{\w \in \mathbb{R}^p : \|\w\|_2\leq B\}$, and $\mathcal{D}$ is the oblique manifold of all matrices with unit-norm columns (or \emph{atoms}); i.e. $\mathcal{D} =\{ \D \in \mathbb{R}^{d\times p} : \|\D_i\|_2 = 1 ~ \forall i \in [p] \}$. 
While not explicit in our notation, $\enc[\D]{(\x)}$ depends on the value of $\lambda$. For notational simplicity, we also suppress subscripts $(\D,\w)$ in $f_{\D,\w}(\cdot)$ and simply write $f(\cdot)$.

We consider a bounded loss function $\ell: \mathcal{Y}\times  \mathcal{Y}' \to [0,b]$, with Lipschitz constant $L_\ell$. 
The goal of learning is to find an $f \in \mathcal{H}$ with  minimal risk, or expected loss, $R(f) = \textstyle\mathbb{E}_{(\x,y)\sim P} \left[ \ell(y,f(\x)) \right]$. Given a sample $S=\{(\x_i,y_i)\}_{i=1}^m$, drawn i.i.d. from $P$, a popular learning algorithm is empirical risk minimization (ERM) which involves finding $f_{\D, \w}$ that solves the following problem: \vspace{-.1cm}
\begin{equation} 
  \min_{\D,\w}~ \frac{1}{m} \sum_{i=1}^{m} \ell(y_i,f_{\D,\w}(\x_i)). \nonumber
\end{equation}\vspace{-.1cm}

\textbf{Adversarial Learning.} In an adversarial setting, we are interested in hypotheses that are robust to adversarial perturbations of inputs. We focus on \emph{evasion attacks}, in which an attack is deployed at test time (while the training samples are not tampered with). As a result, a more appropriate loss that incorporates the robustness to such contamination is the robust loss \citep{madry2017towards},  $\tilde{\ell}_\nu(y,f(\x)) \coloneqq \max_{\v\in\Delta_\nu} ~ \ell(y,f(\x+\v))$,
where $\Delta$ is some subset of $\mathbb{R}^d$ that restricts the {power} of the adversary. Herein we focus on $\ell_2$ norm-bounded corruptions, $\Delta _\nu = \{\v \in \mathbb{R}^d ~ : ~ \|\v\|_2\leq \nu \}$, and denote by $\tilde{R}_S(f) = \frac{1}{m} \sum_{i=1}^{m} \tilde{\ell}_\nu(y_i,f(\x_i))$ the empirical robust risk of $f$ and  $\tilde{R}(f) = \mathbb{E}_{(\x,y)\sim P}[ \tilde{\ell}_\nu(y,f(\x))]$ its population robust risk w.r.t. distribution $P$.

\textbf{Main Assumptions.} We make two general assumptions throughout this work. 
First, we assume that the dictionaries in $\mathcal{D}$ are $s$-incoherent, i.e, they satisfy a restricted isometry property (RIP). More precisely, for any $s$-sparse vector, $\z \in \mathbb R^p$ with $\|\z\|_0 = s$, there exists a minimal constant $\eta_s<1$ so that $\D$ is close to an isometry, i.e.  $(1-\eta_s)\|\z\|^2_2 \leq \|\D\z\|^2_2 \leq (1 + \eta_s) \|\z\|^2_2$. Broad classes of matrices are known to satisfy this property (e.g. sub-Gaussian matrices \citep{foucart2017mathematical}), although empirically computing this constant for a fixed (deterministic) matrix is generally intractable. Nonetheless, this quantity can be upper bounded by the correlation between columns of $\D$, either via mutual coherence \citep{donoho2003optimally} or the Babel function \citep{tropp2003improved}, both easily computed in practice.

Second, we assume that the map $\enc[\D]$ induces a positive \emph{encoder gap} on the computed features. Given a sample $\x\in\mathcal{X}$ and its encoding, $\enc[\D]{(\x)}$, we denote by $\Lambda^{p-s}$ the set of atoms of cardinality $(p-s)$, i.e., $\Lambda^{p-s} = \{ \mathcal{I} \subseteq \{1,\dots,p\} : |\mathcal{I}|=p-s \}$. The encoder gap $\tau_s(\cdot)$ induced by $\enc[\D]$ on any sample $\x$ is defined \citep{mehta2013sparsity} as 
\begin{equation}
   \tau_s(\x) \coloneqq \max_{\mathcal{I}\in\Lambda^{p-s}} \min_{i \in \mathcal{I}} ~\left( \lambda - | \langle \D_i , \x - \D\enc[\D](\x) \rangle | \right).  \nonumber
\end{equation}
{An equivalent and conceptually simpler definition for $\tau_s(\x)$ is the $(s+1)^{th}$ smallest entry in the  vector $\lambda\mathbf{1} - | \langle \D , \x - \D\enc[\D](\x) \rangle |$. Intuitively, this quantity can be viewed as a measure of maximal energy along any dictionary atom that is not in the support of an input vector.} 
More precisely, recall from the optimality conditions of Problem \eqref{eqn:enc-1} that $|\D_i^T(\x-\D\enc[\D]{(\x)})|=\lambda$ if $[\enc[\D]{(\x)}]_i\neq 0$, and $|\D_i^T(\x-\D\enc[\D]{(\x)})|\leq\lambda$ otherwise. Therefore, if $\tau_s$ is large, this indicates that there exist a set $\mathcal{I}$ of $(p-s)$ atoms that are \emph{far} from entering the support of $\enc[\D]{(\x)}$.
If $\enc[\D](\x)$ has exactly $k$ non-zero entries, we may choose some $s>k$ to obtain $\tau_s(\x)$. In general, $\tau_s(\cdot)$ depends on the energy of the residual, $\x-\D\enc[\D]{(\x)}$, the correlation between the atoms, the parameter $\lambda$, and the cardinality $s$. %As we will shortly see, larger values of $s$ provide larger encoder gaps. 
In a nutshell, if a dictionary $\D$ provides a quickly decaying approximation error as a function of the cardinality $s$, then a positive encoder gap exists for some $s$.

{We consider dictionaries that induce a positive encoder gap in every input sample from a dataset, and define the minimum such margin as $\tau_s^*:=\min_{i\in[m]} \tau_s(\x_i) > 0.$} 
Such a positive encoder exist for quite general distributions, such as $s$-sparse and approximately sparse signals. However, this definition is more general and it will allow us to avoid making any other stronger distributional assumptions. We now illustrate such the encoder gap with both analytic and numerical examples\footnote{Code to reproduce all of our experiments is made available at \href{https://github.com/Sulam-Group/Adversarial-Robust-Supervised-Sparse-Coding}{our github repository.}}.

\textbf{Approximate $k$-sparse signals} Consider signals $\x$ obtained as $\x = \D\z + \v$, where $\D\in\mathcal D$, $\|\v\|_2\leq\nu$ and $\z$ is sampled from a distribution of sparse vectors with up to $k$ non-zeros, with $k < \frac{1}{3}\left(1+\frac{1}{\mu(\D)}\right)$, where $\mu(\D) = \max_{i\neq j}\langle \D_i , \D_j \rangle$ is the mutual coherence of $\D$. Then, for a particular choice of $\lambda$, we have that $\tau_s(\x) > \lambda - \frac{15\mu\nu}{2}, \forall s>k $. This can be shown using standard results in \citep{tropp2006just}; we defer the proof to the Appendix \ref{supp:EncoderGapExample}. 
Different values of $\lambda$ provide different values of $\tau_s(\x)$. To illustrate this trade-off, we generate synthetic approximately $k$-sparse signals ($k=15$) from a dictionary with 120 atoms in 100 dimensions and contaminate them with Gaussian noise. We then numerically compute the value of $\tau^*_s$ as a function of $s$ for different values of $\lambda$, and present the results in \autoref{fig:encoder_gap_synthetic}.

\textbf{Image data} We now demonstrate that a positive encoder exist for natural images as well. In \autoref{fig:encoder_gap_mnist} we similarly depict the value of $\tau_s(\cdot)$, as a function of $s$, for an encoder computed on MNIST digits and {CIFAR images} (from a validation set) with learned dictionaries (further details in \autoref{Sec:Experiments}).

{In summary, the encoder gap is a measure of the ability of a dictionary to sparsely represent data, and one can induce a larger encoder gap by increasing the regularization parameter or the cardinality $s$. 
As we will shortly see, this will provide us with a a controllable knob in our generalization bound.}

\begin{figure}
\centering
\subcaptionbox{\label{fig:encoder_gap_synthetic}}
{\includegraphics[width = .32\textwidth,trim = 30 0 30 20]{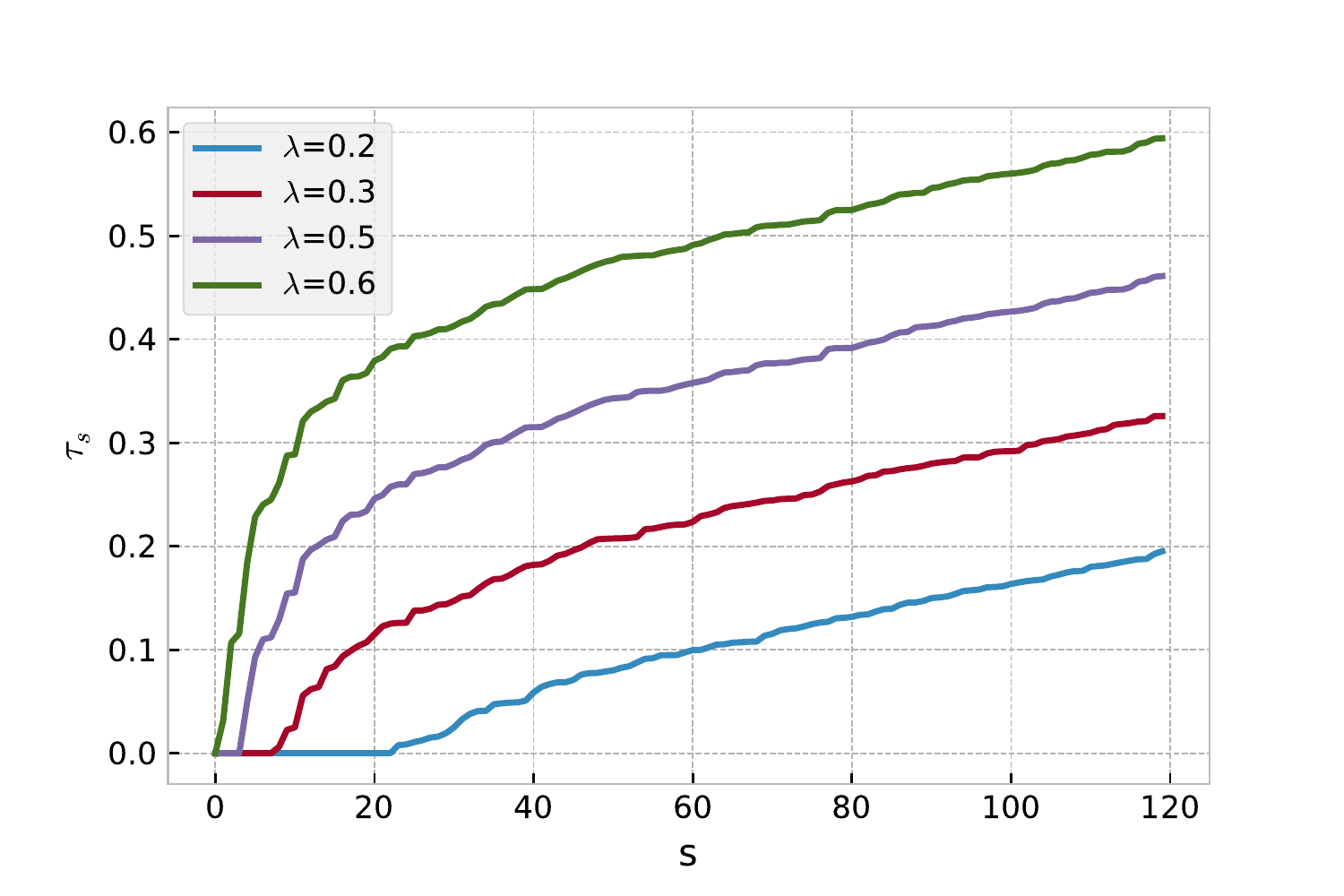}}
\subcaptionbox{\label{fig:encoder_gap_mnist}}
{\includegraphics[width = .32\textwidth,trim = 30 0 30 20]{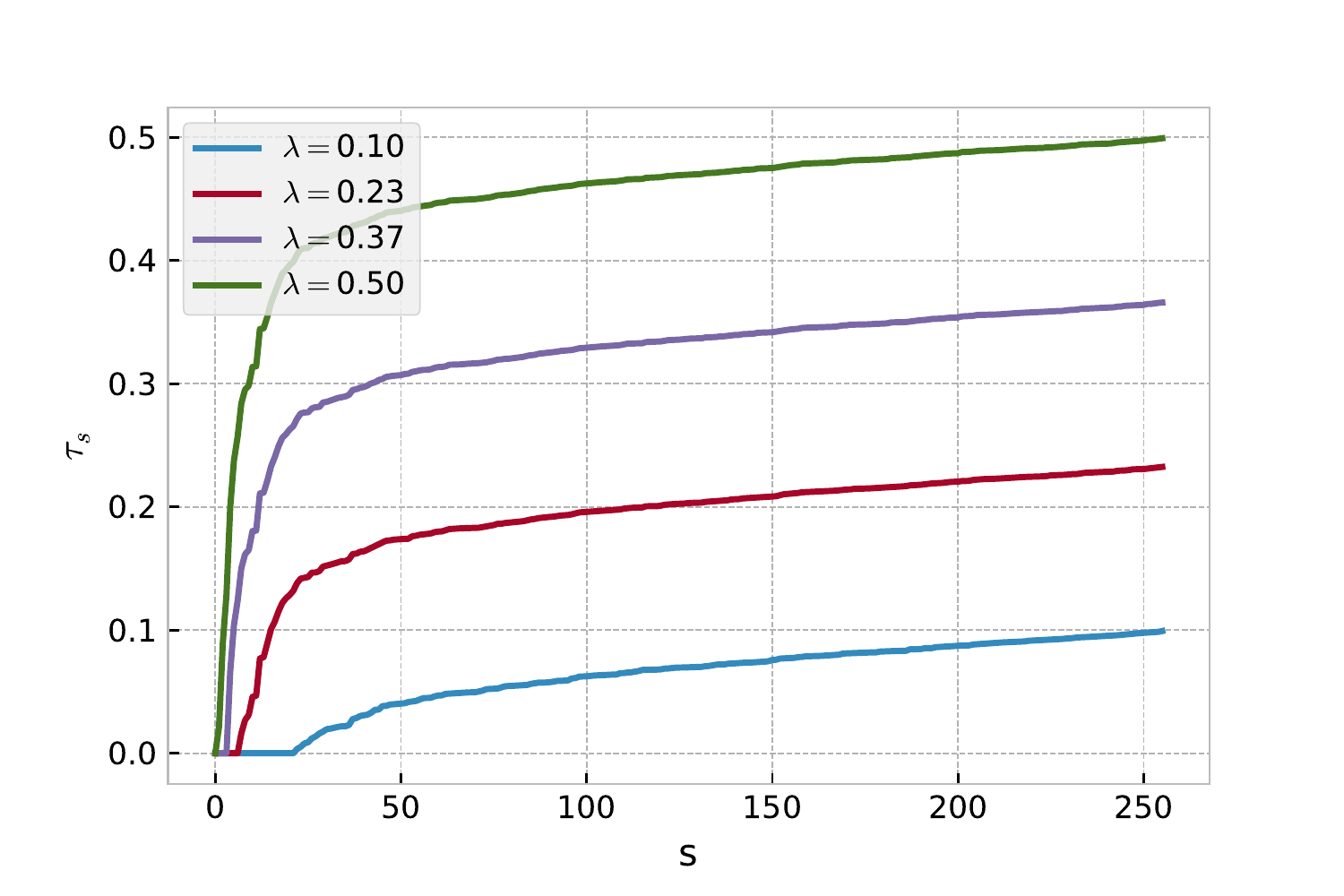}
}
\subcaptionbox{\label{fig:encoder_gap_cifar}}
{\includegraphics[width = .32\textwidth,trim = 30 0 30 20]{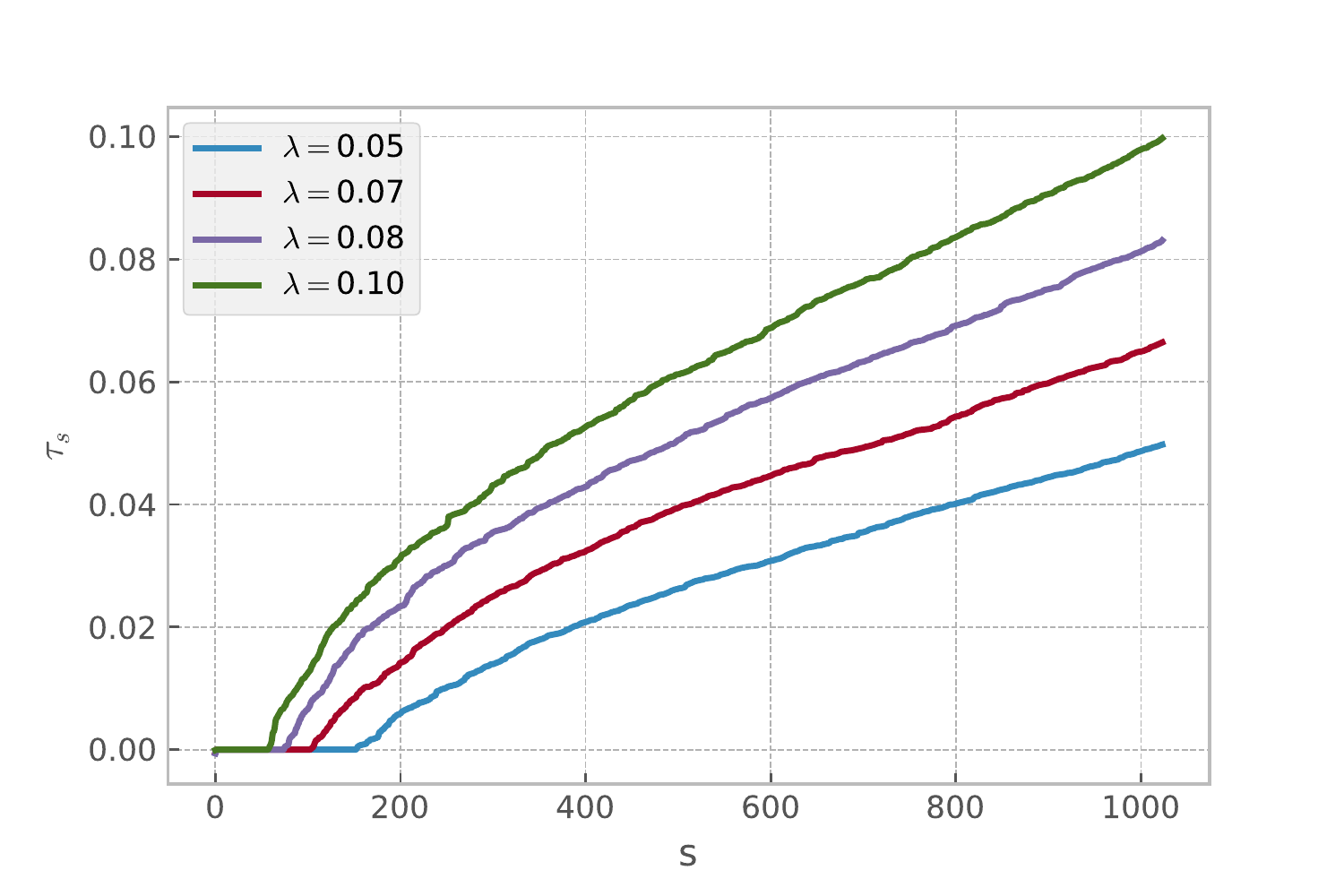}
}
\caption{ Encoder gap, $\tau^*_s$, for synthetic approximately sparse signals (a) MNIST digits (b) and CIFAR10 images (c). \vspace{-.3cm}}
\label{fig:Encoder_Gap}
\end{figure}

\vspace{-5pt}
\section{Prior Work} \label{sec:prior_work}
Many results exist on the approximation power and stability of Lasso (see \citep{foucart2017mathematical}), which most commonly rely on assuming data is (approximately) $k$-sparse under a given dictionary. As explained in the previous section, we instead follow an analysis inspired by \cite{mehta2013sparsity}, which relies on the encoder gap. \cite{mehta2013sparsity} leverage encoder gap to derive a generalization bound for the supervised sparse coding model in a stochastic  (non-adversarial) setting. Their result, which follows a symmetrization technique~\citep{mendelson2004importance}, scales as $\tilde{\mathcal{O}} (\sqrt{(dp + \log(1/\delta))/m}$, and requires  a minimal number of samples that is $\mathcal O (1/(\tau_s\lambda))$. In contrast, we study an generalization in the adversarially robust setting, detailed above. Our analysis is based on an $\epsilon$-cover of the parameter space and on analyzing a local-Lipschitz property of the adversarial loss. The proof of our generalization bound is simpler, and shows a mild deterioration of the upper bound on the generalization gap due to adversarial corruption. 

Our work is also inspired by the line of work initiated by~\cite{papyan2017convolutional} who regard the representations computed by neural networks as approximations for those computed by a Lasso encoder across different layers. In fact, a first analysis of adversarial robustness for such a model is presented by \cite{romano2019adversarial}; however, they make strong generative model assumptions and thus their results are not applicable to real-data practical scenarios. Our robustness certificate mirrors the analysis from the former work, though leveraging a more general and new stability bound (\autoref{lemma:stability_of_rep_adversarial}) relying instead on the existence of positive encoder gap. In a related work, and in the context of neural networks, \cite{cisse2017parseval} propose a regularization term inspired by Parseval frames, with the empirical motivation of improving adversarial robustness. Their regularization term can in fact be related to minimizing the (average) mutual coherence of the dictionaries, which naturally arises as a control for the generalization gap in our analysis.

Lastly, several works have employed sparsity as a beneficial property in adversarial learning \citep{marzi2018sparsity,demontis2016security}, with little or no theoretical analysis, or in different frameworks (e.g. sparse weights in deep networks \citep{guo2018sparse,balda2019adversarial}, or on different domains \citep{bafna2018thwarting}). Our setting is markedly different from that of \cite{chen2013robust} who study adversarial robustness of Lasso as a sparse predictor directly on input features. In contrast, the model we study here employs Lasso as an encoder with a data-dependent dictionary, on which a linear hypothesis is applied. A few works have recently begun to analyze the effect of learned representations in an adversarial learning setting \citep{ilyas2019adversarial,allen2020feature}. Adding to that line of work, our analysis demonstrates that benefits can be provided by exploiting a trade-off between expressivity and stability of the computed representations, and the classifier margin. 

\vspace{-10pt}
\section{Generalization bound for robust risk}
\label{sec:GenBound}
In this section, {we present a bound on the robust risk for models satisfying a positive encoder gap}. Recall that given a $b$-bounded loss $\ell$ with Lipschitz constant $L_\ell$,  $\tilde{R}_S(f) = \frac{1}{m} \sum_{i=1}^{m} \tilde{\ell}_\nu(y_i,f(\x_i))$ is the empirical robust risk, and $\tilde{R}(f) = \mathbb{E}_{(\x,y)\sim P} \big[ \tilde{\ell}_\nu(y,f(\x)) \big]$ is the population robust risk w.r.t. distribution $P$. Adversarial perturbations are bounded in $\ell_2$ norm by $\nu$. {Our main result below guarantees that if a hypothesis $f_{\D,\w}$ is found with a sufficiently large encoder gap, and a large enough training set, its generalization gap is bounded as $\tilde{\mathcal O}\Big(b\sqrt{\frac{ (d+1)p}{m}}\Big)$, where $\tilde{\mathcal O}$ ignores  poly-logarithmic factors.}

\begin{theorem}
\label{thm:generalization} Let $\mathcal W = \{\w \in \mathbb R^p : \|\w\|_2 \leq B \}$,  and $\mathcal{D}$ be the set of column-normalized dictionaries with $p$ columns and with RIP at most $\eta^*_s$. Let $\mathcal{H} = \{ f_{\D,\w}(\x) = \langle \w, \varphi_\D(\x) \rangle : \w\in \mathcal{W}, \D \in \mathcal{D} \}$. Denote $\tau^*_s$ the minimal encoder gap over the $m$ samples. Then, with probability at least $1-\delta$ over the draw of the $m$ samples, the generalization gap for any hypothesis $f\in\mathcal H$ that achieves an encoder gap on the samples of $\tau_s^*>2\nu$, satisfies
\begin{multline}
      \left| \tilde{R}_S(f)\! -\! \tilde{R}(f) \right| \!\leq \frac{b}{\sqrt{m}} \left(  (d+1) p \log\left(\frac{3m}{2\lambda (1-\eta^*_s)}\right) + p\log(B) + \log{\frac{4}{\delta}} \right)^{\frac{1}{2}} \\ + b\sqrt{\frac{2\log(m/2)+2\log(2/\delta)}{m}}  + 12\frac{(1+\nu)^2 L_\ell B \sqrt{s} }{m} , \nonumber 
\end{multline} 
as long as $m > \frac{\lambda(1-\eta_s)}{(\tau_s^* - 2\nu)^2} K_\lambda$,
where $K_\lambda = \left(  2 \left( 1+\frac{1+\nu}{2\lambda} \right) + \frac{5 (1+\nu)}{\sqrt{\lambda}}  \right)^2$. % and $\tilde d = d+1$.
\end{theorem}

A few remarks are in order. First, note that adversarial generalization incurs a polynomial dependence on the adversarial perturbation $\nu$. This is mild, especially since it only affects the fast $\mathcal O(1/m)$ term. Second, the bound requires a minimal number of samples. Such a requirement is intrinsic to the stability of Lasso (see \autoref{lemma:advers_D_stable} below) and it exists also in the non-adversarial setting \citep{mehta2013sparsity}. In the adversarial case, this requirement becomes more demanding, as reflected by the term $(\tau^*_s-2\nu$) in the denominator. Moreover, a minimal encoder gap $\tau_s^*>2\nu$ is needed as well.

\autoref{thm:generalization} suggests an interesting trade-off. One can obtain a large $\tau^*_s$ by increasing $\lambda$ and $s$ -- as demonstrated in in \autoref{fig:Encoder_Gap}. But increasing $\lambda$ may come at an expense of hurting the empirical error, while increasing $s$ makes the term $1-\eta_s$ smaller. {Therefore, if one obtains a model with small training error, along with large $\tau^*_s$ over the training samples for an appropriate choice of $\lambda$ and $s$ while ensuring that $\eta_s$ is bounded away from $1$,  % \emph{blow up} the $\frac{1}{1-\eta_s}$ term, 
then $f_{\D,\w}$ is guaranteed to generalize well. Furthermore, note that the excess error depends mildly (poly logarithmically) on $\lambda$ and $\eta_s$.} 

Our proof technique is based on a minimal $\epsilon$-cover of the parameter space, and the full proof is included in the % Supplementary Material 
Appendix~\ref{supp:Generalization}. Special care is needed to ensure that the encoder gap of the dictionary holds for a sample drawn from the population, as we can only measure this gap on the provided $m$ samples. To address this, we split the data equally into a training set and a development set: the former is used to learn the dictionary, and the latter to provide a high probability bound on the event that $\tau_s(\x)>\tau^*_s$. This is to ensure that the random samples of the encoder margin are i.i.d. for measure concentration. Ideally, we would like to utilize the entire dataset for learning the predictor; we leave that for future work.

While most of the techniques we use are standard~\footnote{See \citep{seibert2019sample} for a comprehensive review on these tools in matrix factorization problems.}, the Lipschitz continuity of the robust loss function requires a more delicate analysis. For that, we have the following result. 

\begin{lemma}[Parameter adversarial stability] \label{lemma:advers_D_stable}
Let $\D,\D' \in \mathcal D$. If $\|\D-\D'\|_2 \leq \epsilon\leq 2 \lambda/(1+\nu)^2$,  then
    \begin{equation} \label{eq:adver_perturb_stability}
       \max_{\v\in\Delta} \|\enc[\D](\x+\v) -  \enc[\D'](\x+\v) \|_2 \leq \gamma (1+\nu)^2 \epsilon,
    \end{equation}
    with $\gamma = \frac{3}{2} \frac{\sqrt{s}}{\lambda(1-\eta_s)}$, as long as $\tau_s(\x) \geq 2 \nu +  \sqrt{\epsilon}\left( \sqrt{\frac{25}{\lambda}}(1+\nu) + 2 \left( \frac{(1+\nu)}{\lambda} + 1\right) \right).$
\end{lemma}
\autoref{lemma:advers_D_stable} is central to our proof, as it provides a bound on difference between the features computed by the encoder under model deviations. Note that the condition on the minimal encoder gap, $\tau_s(\x)$, puts an upper bound on the distance between models $\D$ and $\D'$. This in turn results in the condition imposed on the minimal samples in \autoref{thm:generalization}. It is worth stressing that the lower bound on $\tau_s(\x)$ is on the \emph{unperturbed} encoder gap -- that which can be evaluated on the samples from the dataset, without the need of the adversary. We defer the proof of this Lemma to
Appendix~\ref{supp:ParameterAdversarialStability}.

\section{Robustness Certificate}
\label{sec:certificates}

Next, we turn to address an equally important question about robust adversarial learning, that of giving a formal certification of robustness. Formally, we would like to guarantee that the output of the trained model, $f_{\D,\w}(\x)$, does not change for norm-bounded adversarial perturbations of a certain size. 
Our second main result provides such a certificate for the supervised sparse coding model. 

Here, we consider a multiclass classification setting with $y\in\{1,\dots,K\}$; simplified results for binary classification are included in Appendix \ref{supp:certificate}. The hypothesis class is parameterized as $f_{\D,\W}(\x) = \W^T\enc[\D](\x)$, with $\W = [\W_1, \W_2, \ldots, \W_K] \in\mathbb{R}^{p\times K}$. The multiclass margin is defined as follows: 
\[ \rho_\x = \W^T_{y_i} \enc[\D](\x) - \max_{j\neq y_i} \W^T_{j} \enc[\D](\x). \]
We show the following result. 
\begin{theorem}[Robustness certificate for multiclass supervised sparse coding] \label{thm:multiclass_certificate}
Let $\rho_x > 0$ be the multiclass classifier margin of $f_{\D,\w}(\x)$ composed of an encoder with a gap of $\tau_s(\x)$ and a dictionary, $\D$, with RIP constant $\eta_s<1$. Let $c_\W := \max_{i\neq j}\|\W_i-\W_j\|_2$. Then,
    \begin{equation}
        \arg\max_{j\in[K]} ~[\W^T \enc[\D](\x)]_j\ = \arg\max_{j\in[K]}~ [\W^T \enc[\D](\x+\v)]_j,\quad \forall~ \v:\|\v\|_2\leq \nu,  
    \end{equation}
    so long as $\nu \leq \min\{ \tau_s(\x)/2 , \rho_\x \sqrt{1-\eta_s} / c_\W \}.$
\end{theorem}

\autoref{thm:multiclass_certificate} clearly captures the potential contamination on two flanks: robustness can no longer be guaranteed as soon as the energy of the perturbation is enough to either significantly modify the computed representation \emph{or} to induce a perturbation larger than the classifier margin on the feature space. 
Proof of \autoref{thm:multiclass_certificate}, detailed in Appendix \ref{supp:certificate}, relies on the following lemma showing that under an encoder gap assumption, the computed features are moderately affected despite adversarial corruptions of the input vector. 
\begin{lemma}[Stability of representations under adversarial perturbations] \label{lemma:stability_of_rep_adversarial}
Let $\D$ be a dictionary with RIP constant $\eta_s$. Then, for any $\x \in \mathcal X$ and its additive perturbation $\x+\v$, for any $\|\v\|_2\leq \nu$, if $\tau_s(\x) >  2 \nu$, then we have that
\begin{equation}
    \| \enc[\D]{(\x)} - \enc[\D]{(\x+\v)} \|_2 \leq \frac{\nu}{\sqrt{1-\eta_s}}.
\end{equation}
\end{lemma}
An extensive set of results exist for the stability of the solution provided by Lasso relying generative model assumptions \citep{foucart2017mathematical,elad2010sparse}. The novelty of Lemma~\ref{lemma:stability_of_rep_adversarial} is in replacing such an assumption with the existence of a positive encoder gap on $\enc[\D](\x)$.

Going back to \autoref{thm:multiclass_certificate}, note that the upper bound on $\nu$ depends on the RIP constant $\eta_s$, which is not computable for a given (deterministic) matrix $\D$. Yet, this result can be naturally relaxed by upper bounding $\eta_s$ with measures of correlation between the atoms, such as the mutual coherence. This quantity provides a measure of the worst correlation between two atoms in the dictionary $\D$, and it is defined as $\mu(\D) = \max_{i\neq j} | \langle \D_i , \D_j \rangle|$ (for $\D$ with normalized columns). For general (overcomplete and full rank) dictionaries, clearly $0<\mu(\D)\leq 1$. 

While conceptually simple, results that use $\mu(\D)$ tend to be too conservative. Tighter bounds on $\eta_s$ can be provided by the Babel function\footnote{Let $\Lambda$ denote subsets (supports) of $\{1,2,\dots,p\}$. Then, the Babel function is defined as \mbox{$\mu_{(s)} = \max_{\Lambda:|\Lambda|=s} \max_{j\notin \Lambda} \sum_{i\in\Lambda} |\langle \D_i , \D_j \rangle|$.}}, $\mu_{(s)}$, which quantifies the maximum correlation between an atom and \emph{any other} collection of $s$ atoms in $\D$. It can be shown \citep[Chapter 2]{tropp2003improved,elad2010sparse} that $\eta_s\leq \mu_{(s-1)}\leq(s-1)\mu(\D)$. Therefore, we have the following:
\begin{corollary} \label{cor:certificate_babel_multiclass}
Under the same assumptions as those in \autoref{thm:multiclass_certificate}, 
\begin{equation}
       \arg\max_{j\in[K]} ~[\W^T \enc[\D](\x)]_j\ = \arg\max_{j\in[K]}~ [\W^T \enc[\D](\x+\v)]_j,\quad \forall~ \v:\|\v\|_2\leq \nu 
    \end{equation}
    so long as $\nu \leq \min\{ \tau_s(\x)/2 , \rho_\x \sqrt{1-\mu_{(s-1)}}/c_\W \}$.
\end{corollary}
Although % we have made 
the condition on $\nu$ in the corollary above is stricter (and the bound looser), the quantities involved can easily be computed numerically leading to practical useful bounds, as we see~next.

\vspace*{-5pt}
\section{Experiments}
\label{Sec:Experiments}

In this section, we illustrate the robustness certificate guarantees % described above, 
both in synthetic and real data, as well as the trade-offs between constants in our sample complexity result. First, we construct samples from a separable binary distribution of $k$-sparse signals. To this end, we employ a dictionary with 120 atoms in 100 dimensions with a mutual coherence of 0.054. Sparse representations $\z$ are constructed by first drawing their support (with cardinality $k$) uniformly at random, and drawing its non-zero entries from a uniform distribution away from zero. Samples are obtained as $\x = \D\z$, and normalized to unit norm. We finally enforce separability by drawing $\w$ at random from the unit ball, determining the labels as $y=\text{sign}( \w^T\enc[\D]{(\x)})$, and discarding samples with a margin $\rho$ smaller than a pre-specified amount (0.05 in this case). Because of the separable construction, the accuracy of the resulting classifier is 1.

We then attack the obtained model employing the projected gradient descent method \citep{madry2017towards}, and analyze the degradation in accuracy as a function of the energy budget $\nu$. We compare this empirical performance with the bound in  \autoref{cor:certificate_babel_multiclass}: given the obtained margin, $\rho$, and the dictionary's $\mu_{s}$, we can compute the maximal certified radius for a sample $\x$ as
\begin{equation} \label{eq:computing_bound}
\nu(\x) = \max_s \min\{ \tau_s(\x)/2 , \rho_\x \sqrt{1-\mu_{(s-1)}}/c_\W\}.
\end{equation}
{For a given dataset, we can compute the minimal certified radius over the samples, $\nu^* = \min_{i\in[n]} \nu(\x_i)$.} This is the bound depicted in the vertical line in \autoref{fig:synthetic_separable}. As can be seen, despite being somewhat loose, the attacks do not change the label of the samples, thus preserving the accuracy.

In non-separable distributions, one may study how the accuracy depends on the \emph{soft margin} of the classifier. In this way, one can determine a target margin that results in, say, 75\% accuracy on a validation set. One can obtain a corresponding certified radius of $\nu^*$ as before, which will guarantee that the accuracy will not drop below 75\% as long as $\nu<\nu^*$. This is illustrated in \autoref{fig:synthetic_non_separable}. 

An alternative way of employing our results from \autoref{sec:certificates} is by studying the \emph{certified accuracy} achieved by the resulting hypothesis. 
The certified accuracy quantifies the percentage of samples in a test set that are classified correctly while being \emph{certifiable}. In our context, this implies that a sample $\x$ achieves a margin of $\rho_\x$, for which a certified radius of $\nu^*$ can be obtained with \autoref{eq:computing_bound}. In this way, one may study how the certified accuracy decreases with increasing $\nu^*$. 

This analysis lets us compare our bounds with those of other popular certification techniques, such as randomized smoothing \citep{cohen2019certified}. Randomized smoothing provides high probability robustness guarantees against $\ell_2$ attacks for \emph{any} classifier by composing them with a Gaussian distribution (though other distributions have been recently explored as well for other $l_p$ norms \citep{salman2020black}). In a nutshell, the larger the variance of the Gaussian, the larger the certifiable radius becomes, albeit at the expense of a drop in accuracy.

We use the MNIST dataset for this analysis. We train a model with 256 atoms by minimizing the following regularized empirical risk using stochastic gradient descent (employing Adam \citep{kingma2014adam}; the implementation details are deferred to Appendix \ref{supp:numerical}) 
\begin{equation}
    \min_{\W,\D}~ \frac{1}{m} \sum_{i=1}^m \ell\left(y_i,\langle \W,\enc[\D]{(\x_i)} \rangle\right) + \alpha \|\I - \D^T\D\|^2_F + \beta \|\W\|^2_F,
\end{equation}
where $\ell$ is the cross entropy loss. 
Recall that $\enc[\D](\x)$ depends on $\lambda$, and we train two different models with two values for this parameter ($\lambda=0.2$ and $\lambda=0.3$).

\begin{figure}
\centering
\subcaptionbox{\label{fig:synthetic_separable}}
{\hspace*{-2pt}\includegraphics[width = .25\textwidth,trim = 0 0 0 0 0]{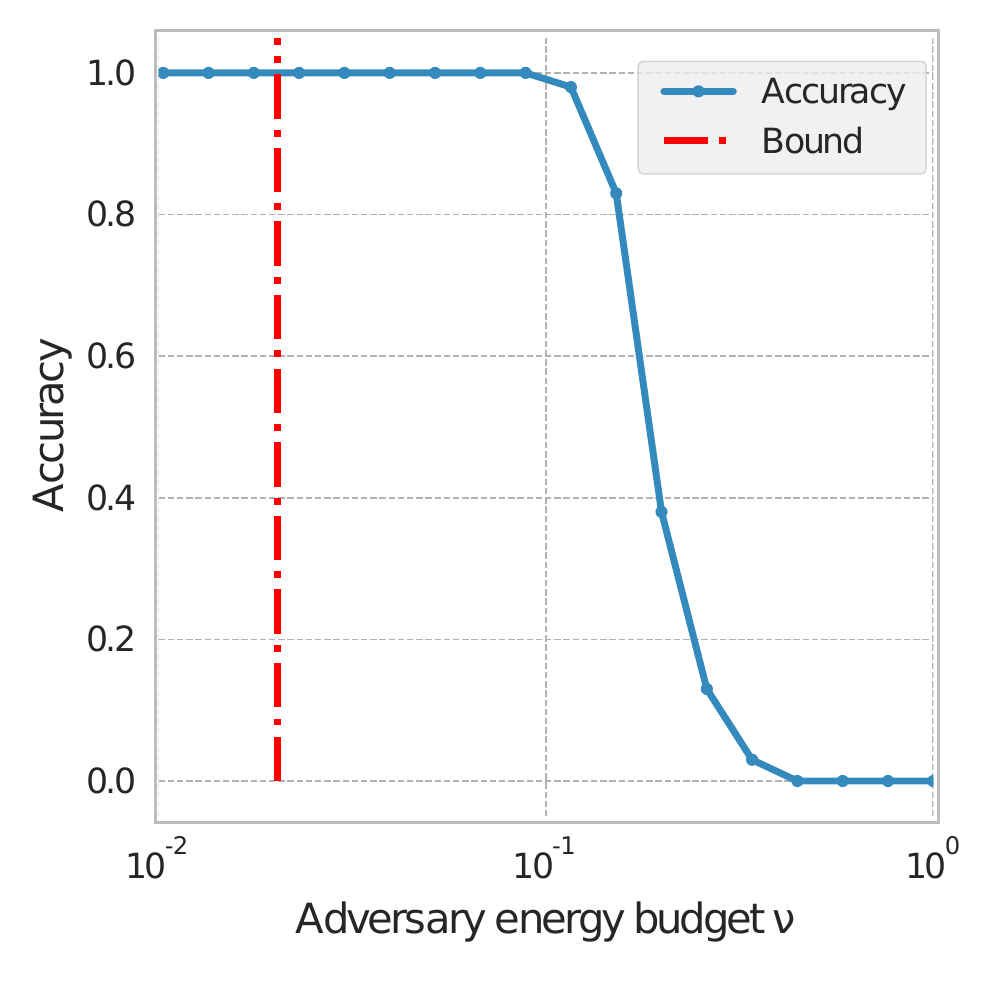}}
\subcaptionbox{\label{fig:synthetic_non_separable}}
{\hspace*{-12pt}\includegraphics[width = .25\textwidth,trim = 0 0 0 0]{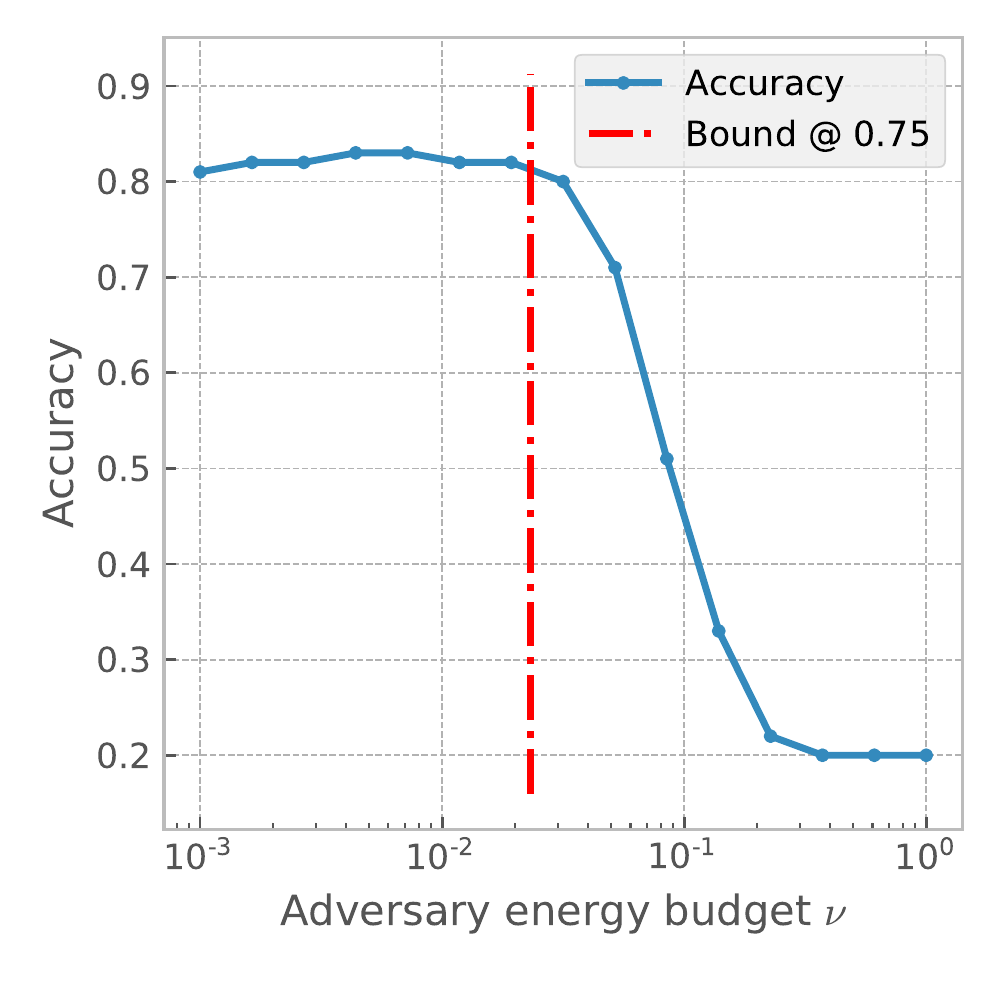}}
\subcaptionbox{\label{fig:mnist_comparison_a}}
{\hspace*{3pt}\includegraphics[width = .25\textwidth,trim = 0 -5 20 30]{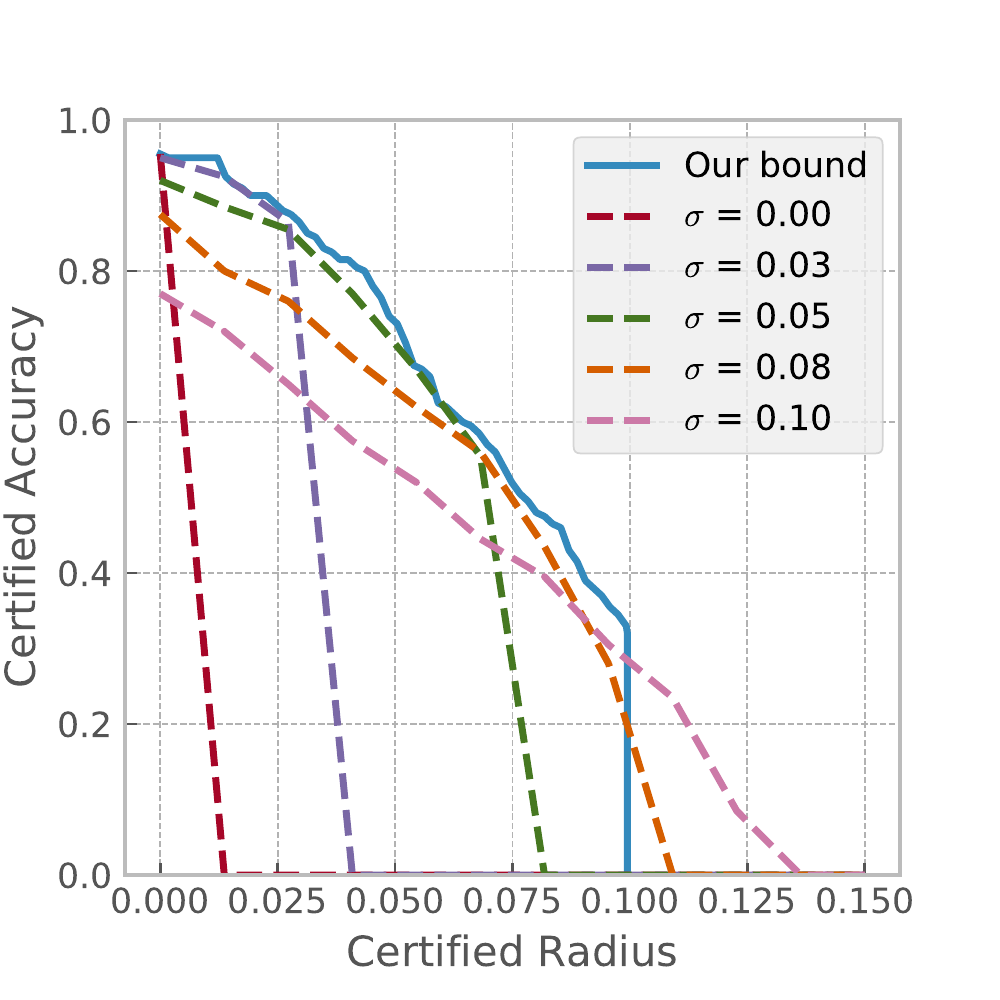}}
\subcaptionbox{\label{fig:mnist_comparison_b}}
{\hspace*{3pt}\includegraphics[width = .25\textwidth,trim = 0 -5 20 30]{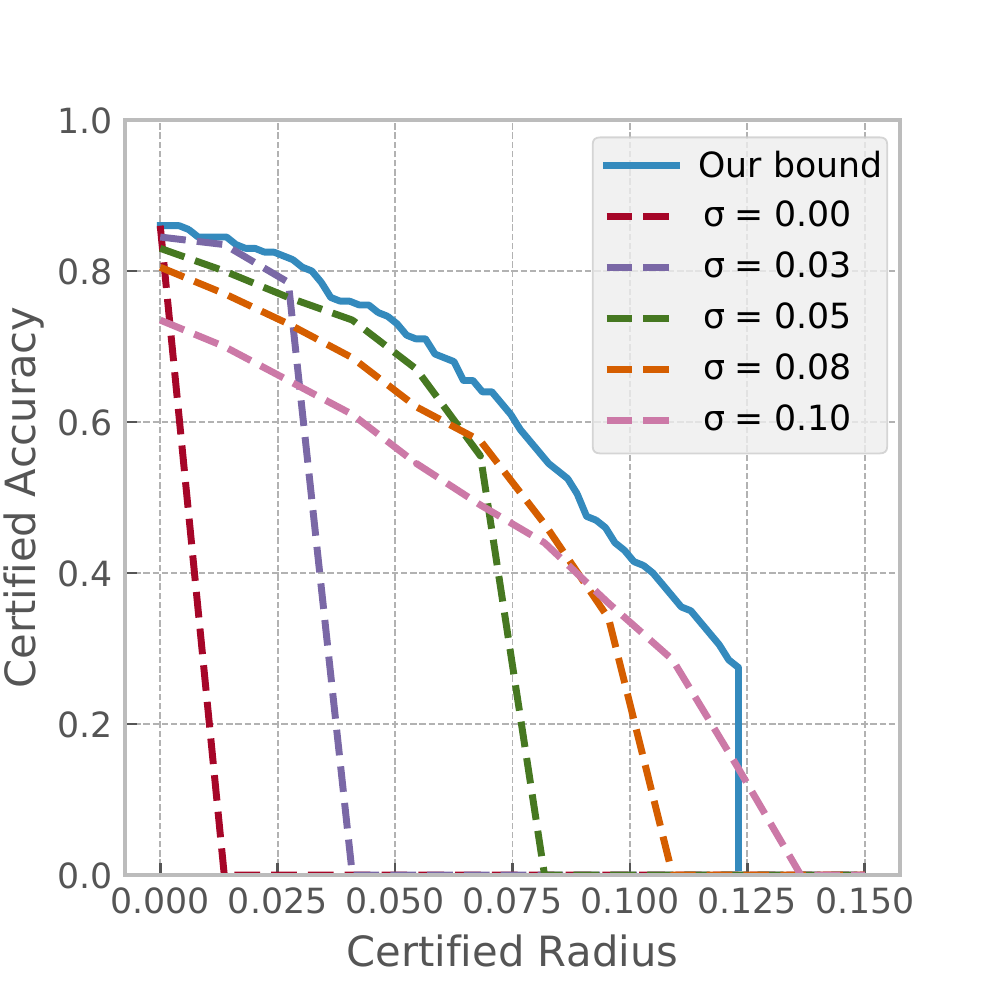}}
\caption{Numerical demonstrations of our results. (a) synthetic separable distribution. (b) synthetic non-separable distribution. (c-d) certified accuracy on MNIST with $\lambda = 0.2$ and $\lambda = 0.3$, respectively, comparing with Randomized Smoothing with different variance levels.\vspace{-8pt}}
\label{fig:Certificates}
\end{figure}
\autoref{fig:mnist_comparison_a} and  \ref{fig:mnist_comparison_b} illustrate the certified accuracy on 200 test samples obtained by different degrees of randomized smoothing and by our result. While the certified accuracy resulting from our bound is comparable to that by randomized smoothing, the latter provides a certificate by \emph{defending} (i.e. composing it with a Gaussian distribution). In other words, different \emph{smoothed} models have to be constructed to provide different levels of certified accuracy. In contrast, our model is not defended or modified in any way, and the certificate relies solely on our careful characterization of the function class. 
Since randomized smoothing makes no assumptions about the model, the bounds provided by this strategy rely on the estimation of the output probabilities. This results in a heavy computational burden to provide a high-probability result (a failure probability of 0.01\% was used for these experiments). In contrast, our bound is deterministic and trivial to compute. 

Lastly, comparing the results in \autoref{fig:mnist_comparison_a} (where $\lambda = 0.2$) and \autoref{fig:mnist_comparison_b} (where $\lambda = 0.3$), we see the trade-off that we alluded to in \autoref{sec:GenBound}: larger values of $\lambda$ allow for larger encoder gaps, resulting in overall larger possible certified radius. In fact, $\lambda$ determines a hard bound on the possible achieved certified radius, given by $\lambda/2$, as per \autoref{eq:computing_bound}. This, however, comes at the expense of reducing the complexity of the representations computed by the encoder $\enc[\D]{(\x)}$, which impacts  the risk attained. 

\vspace{-5pt}
\section{Conclusion}
\label{sec:Conclusion}

In this paper we study the adversarial robustness of the supervised sparse coding model from two main perspectives: {we provide a bound for the robust risk of any hypothesis that achieves a minimum encoder gap over the samples, as well as a robustness certificate for the resulting end-to-end classifier}. Our results describe guarantees relying on the interplay between the computed representations, or features, and the classifier margin.

While the model studied is still relatively simple, we envision several ways in which our analysis can be extended to more complex models. First, high dimensional data with shift-invariant properties (such as images) often benefit from convolutional features. Our results do hold for convolutional dictionaries, but the conditions on the mutual coherence may become prohibitive in this setting. An analogous definition of encoder gap in terms of convolutional sparsity \citep{papyan2017working} may provide a solution to this limitation. Furthermore, this analysis could also be extended to sparse models with multiple layers, as in \citep{papyan2017convolutional,sulam2019multi}. {On the other hand, while our result does not provide a uniform learning bound over the hypothesis class, we have found empirically that regularized ERM does indeed return hypotheses satisfying non-trivial encoder gaps. The theoretical underpinning of this phenomenon needs further research.}
More generally, even though this work focuses on sparse encoders, we believe similar principles could be generalized to other forms of representations in a supervised learning setting, providing a framework for the principled analysis of adversarial robustness of machine learning models.

\section*{Broader Impact}
This work contributes to the theoretical understanding of the limitations and achievable robustness guarantees for supervised learning models. Our results can therefore provide tools that could be deployed in sensitive settings where these types of guarantees are a priority. On a broader note, this work advocates for the precise analysis and characterization of the data-driven features computed by modern machine learning models, and we hope our results facilitate their generalization to other more complex models. % Besides these, we do not foresee any other immediate societal consequence.

\section*{Acknowledgements} { 
This research was supported, in part, by DARPA GARD award HR00112020004, NSF BIGDATA award IIS-1546482, NSF CAREER award IIS-1943251 and NSF TRIPODS award CCF-1934979. Jeremias Sulam kindly thanks Aviad Aberdam for motivating and inspiring discussions. Raman Arora acknowledges support from the 
Simons Institute as part of the program on the Foundations of Deep Learning and the 
Institute for Advanced Study (IAS), Princeton, NJ, as part of the special year on Optimization, Statistics, and Theoretical Machine Learning. 
}
\medskip 
% \small
\bibliographystyle{plainnat}
\bibliography{refs}

\newpage
\appendix

\begin{center}
    \Large \textbf{Supplementary Material for \\ Adversarial Robustness of Supervised Sparse Coding}
\end{center}

\section{Encoder Gap for $k$-sparse signals}
\label{supp:EncoderGapExample}

Herein we show that a positive encoder gap exists for signals that are (approximately) $k$-sparse. Consider signals $\x$ obtained as $\x = \D\z + \v$, where $\D\in\mathcal D$, $\|\v\|_2\leq\nu$ and $\z$ is sampled from a distribution of sparse vectors with up to $k$ non-zeros, with $k < \frac{1}{3}\left(1+\frac{1}{\mu(\D)}\right)$, where $\mu(\D) = \max_{i\neq j}\langle \D_i , \D_j \rangle$ is the mutual coherence of $\D$. Then, from \citep{tropp2006just}, if $\lambda=4\nu$, the (unique) solution recovered by $\alfa = \enc[\D]{(\x+\v)}$ satisfies $\|\alfa-\z\|_\infty\leq \frac{15}{2}\nu$, and $Supp(\alfa)\subseteq Supp(\z)$. Recall the definition of encoder gap:
\[\tau_s(\x) \coloneqq \max_{\mathcal{I}\in\Lambda^{p-s}} \min_{i \in \mathcal{I}} ~\left( \lambda - | \langle \D_i , \x - \D\enc[\D](\x) \rangle | \right)\]
and pick $s>k$. Let $S = Supp(\z)$. Thus, the maximization over $I$ is achieved by a subset $\mathcal I$ which does not contain any of the active atoms in $\z$ (for which $| \langle \D_i , \x - \D\enc[\D](\x) \rangle |=\lambda$, by optimality).

Now, define $\Delta = \alfa - \z$, and let $\Delta_S$ denote the vector $\Delta$ restricted to the support $S$ and $\D_S$ the sub-dictionary obtained $\D$ by restricting it to the same set of atoms. We can then write
\begin{align}
    \max_i \left| \D_i^T (\x-\D\alfa ) \right| &= \max_i \left| \D_i^T (\x-\D\z ) - \D_i^T\D \Delta \right|\\
    &= \max_i \left| \D_i^T \D_S \Delta_S \right| \\
    &\leq \max_i | \D_i^T \D_S | \|\Delta_S\|_\infty \\
    &= \frac{15}{2} \mu(\D) \nu,
    \end{align}
because $i\notin S$. Thus,
\begin{equation}
    \min_i \lambda - | \langle \D_i , \x - \D\enc[\D](\x) \rangle | \geq \lambda - \frac{15}{2} \mu(\D) \nu.
\end{equation}
In fact, recalling that $\lambda = 4\nu$, we have that $\tau_s\geq\nu (4-\mu(\D) \frac{15}{2})$.

\section{Robust Generalization Bound}
\label{supp:Generalization}
Herein we prove our generalization bound, but first re-state it for completeness.

\setcounter{section}{4}
\renewcommand{\thesection}{\arabic{section}}
\setcounter{theorem}{0}
\begin{theorem}
Let $\mathcal W = \{\w \in \mathbb R^p : \|\w\|_2 \leq B \}$,  and $\mathcal{D}$ be the set of column-normalized dictionaries with $p$ columns and with RIP at most $\eta^*_s$. Let $\mathcal{H} = \{ f_{\D,\w}(\x) = \langle \w, \varphi_\D(\x) \rangle : \w\in \mathcal{W}, \D \in \mathcal{D} \}$. Denote $\tau^*_s$ the minimal encoder gap over the $m$ samples. Then, with probability at least $1-\delta$ over the draw of the $m$ samples, the generalization gap for any hypothesis $f\in\mathcal H$ that achieves an encoder gap of $\tau_s^*>2\nu$, satisfies
\begin{multline}
       \left| \tilde{R}_S(f)\! -\! \tilde{R}(f) \right| \!\leq \frac{b}{\sqrt{m}} \left(  (d+1) p \log\left(\frac{3m}{2\lambda (1-\eta^*_s)}\right) + p\log(B) + \log{\frac{4}{\delta}} \right)^{\frac{1}{2}} \\ + b\sqrt{\frac{2\log(m/2)+2\log(2/\delta)}{m}}  + 12\frac{(1+\nu)^2 L_\ell B \sqrt{s} }{m} , \nonumber 
\end{multline} 
as long as $m > \frac{\lambda(1-\eta_s)}{(\tau_s^* - 2\nu)^2} K_\lambda$,
where $K_\lambda = \left(  2 \left( 1+\frac{1+\nu}{2\lambda} \right) + \frac{5 (1+\nu)}{\sqrt{\lambda}}  \right)^2$.
\end{theorem}

\setcounter{section}{2}
\renewcommand{\thesection}{\Alph{section}}
 
\begin{proof}
    Fix $\epsilon>0$, and consider a  minimal $\epsilon$-cover for the parameter space $(\mathcal{D},\mathcal{W})$ with respect to a metric $d$ and with the elements $(\D_j,\w_j)$,  $j\in\{1,\dots,N^{cov}((\mathcal{D},\mathcal{W}),\epsilon)\}$. The metric we consider is the max over the operator norm and $\ell_2$ norm on $\mathcal{D}$ and $\mathcal{W}$, respectively, i.e. $d\left( (\D,\w), (\D',\w')\right) = \max \{ \|\D-\D'\|_2, \|\w-\w'\|_2 \}$. 
    Now, fixing $(\D,\w)$, by the definition of the $\epsilon$-cover, there exist an index $j$ so that $d( (\D_j,\w_j),(\D,\w) )\leq \epsilon$. We thus expand the generalization gap into three terms, as follows:
    \begin{align} 
         \left| \tilde{R}_S(f)\! -\! \tilde{R}(f) \right| & =  \left| ~ \frac{1}{m} \sum_{i=1}^{m} \tilde{\ell}_\nu(y_i,f(\x_i)) ~~- \underset{(\x,y)\sim P}{\mathbb{E}} \left[ \tilde{\ell}_\nu(y,f(\x)) \right] \right| \\ \label{eq:3_parts}
          & \leq \sup_{k\in [N^\text{cov}]} \Big| \frac{1}{m} \sum_{i=1}^{m} \tilde{\ell}_\nu(y_i,f_{\D_k,\w_k}(\x_i)) - \underset{(\x,y)}{\mathbb{E}} \left[ \tilde{\ell}_\nu(y,f_{\D_k,\w_k}(\x)) \right] \Big| \\
          &\quad +  \Big| \frac{1}{m} \sum_{i=1}^{m} \tilde{\ell}_\nu(y_i,f_{\D,\w}(\x_i)) - \frac{1}{m} \sum_{i=1}^{m} \tilde{\ell}_\nu(y_i,f_{\D_j,\w_j}(\x_i)) \Big| \\
          &\quad +  \Big| \underset{(\x,y)}{\mathbb{E}} \left[ \tilde{\ell}_\nu(y,f_{\D_j,\w_j}(\x)) \right] - \underset{(\x,y)}{\mathbb{E}} \left[ \tilde{\ell}_\nu(y,f_{\D,\w}(\x)) \right]  \Big|. 
    \end{align}{}
    
    Let us bound the first of these terms. Let $\z_i$ denote the random tuple $(y_i,\x_i)$, and $\Z = \left[ (y_1,\x_1),\dots, (y_m,\x_m) \right]$. Let $g(\Z) = \frac{1}{m}\sum_{i=1}^m \tilde{\ell}_f(y_i,\x_i)$. Furthermore, consider $\Z'$ as the set of $m$ random variables $\z$ that only differs from $\Z$ in its $i^{th}$ variable, $\z_i' = (y'_i,\x_i')$. Then, for any $i \in [m]$,
    \begin{equation}
        | g(\Z) - g(\Z') | = \left| \frac{1}{m} \left( \tilde{\ell}_{\nu_f}(y_i,\x_i) - \tilde{\ell}_{\nu_f}(y'_i,\x'_i) \right) \right| \leq \frac{b}{m},
    \end{equation}
    since $\ell(y,f(\x))$, and thus $\tilde\ell(y,f(\x))$, is bounded.  
    \begin{equation}
         \text{Pr} \left[ ~ |  g(\Z) - \mathbb E [ g(\Z) ] | \geq t ~ \right] \leq 2\exp\left(\frac{-2mt^2}{b^2} \right).
    \end{equation}
    Furthermore, note that $\mathbb E [ g(\Z) ] = \mathbb E [ \tilde{\ell}_{\nu_f}(y,\x) ] $ (linearity of expectation), and thus we have that
    \begin{equation}
        \text{Pr}\Bigg[ \Big| \frac{1}{m} \sum_{i=1}^{m} \tilde{\ell}_\nu(y_i,f_{\D,\w}(\x_i)) - \underset{(\x,y)}{\mathbb{E}}  \tilde{\ell}_\nu(y,f_{\D,\w}(\x)) \Big| > t \Bigg] \leq 2 \exp{\left(\frac{-2mt^2}{b^2}\right)}.
    \end{equation}
    
    Next, using a union bound argument, we can bound the probability over the supremum:
    \begin{multline} \label{eq:fail_1}
        \text{Pr}\Bigg[ \sup_{j} \Big| \frac{1}{m} \sum_{i=1}^{m} \tilde{\ell}(y_i,f_{\D_j,\w_j}(\x_i)) - \underset{(\x,y)}{\mathbb{E}}  \tilde{\ell}(y,f_{\D_j,\w_j}(\x)) \Big| > t \Bigg] \leq \\ \sum_{j=1}^{N^{cov}} \text{Pr}\Bigg[ \Big| \frac{1}{m} \sum_{i=1}^{m} \tilde{\ell}(y_i,f_{\D_j,\w_j}(\x_i)) - \underset{(\x,y)}{\mathbb{E}}  \tilde{\ell}(y,f_{\D_j,\w_j}(\x)) \Big| > t \Bigg] \leq 2 N^{cov} \exp{\left(\frac{-2mt^2}{b^2}\right)}.
    \end{multline}
    
    Denote this failure probability as $\delta/2$. Thus, with probability at least $1-\delta/2$, we get 
    \begin{equation}
        \sup_{j} \Big| \frac{1}{m} \sum_{i=1}^{m} \tilde{\ell}(y_i,f_{\D_j,\w_j}(\x_i)) - \underset{(\x,y)}{\mathbb{E}} \left[ \tilde{\ell}(y,f_{\D_j,\w_j}(\x)) \right] \Big| \leq b \sqrt{ \frac{\log(N^{cov}) + \log(4/\delta) }{2m}}.
    \end{equation}\\
    
    Let us now focus on the second and third terms in Eq. \eqref{eq:3_parts}. In particular, we will upper bound them by analyzing the Lipschitz continuity of the loss function with respect to the parameters, $\D$ and $\w$. 
    We assume that $\ell$ is $L_\ell$-Lipschitz, and we analyze the Lipschitz continuity of $\tilde{\ell}$ w.r.t $\D$ through $f_\D(\x)$. Noting that the difference of the maxima is upper-bounded by the maximum of the difference, we can write
    \begin{align}
        \big| \tilde{\ell}(y,f_{\D}(\x)) - \tilde{\ell}(y,f_{\D'}(\x)) \big| &= \big|\max_{\v\in\Delta} \ell(y,f_{\D}(\x+\v)) - \max_{\v\in\Delta} \ell(y,f_{\D'}(\x+\v)) \big| \\
        &\leq \max_{\v\in\Delta}  \big| \ell(y,f_{\D}(\x+\v))  - \ell(y,f_{\D'}(\x+\v)) \big| \\
        & \leq   L_\ell \max_{\v\in\Delta} | \langle \w^T, \enc[\D](\x+\v) \rangle  - \langle \w^T, \enc[\D'](\x+\v) \rangle  | \\ \label{eq:lip_bound_1}
        & \leq L_\ell \|\w\|_2  \max_{\v\in\Delta} \|\enc[\D](\x+\v) - \enc[\D'](\x+\v) \|_2.
    \end{align}
    
    We will now bound the term $\max_{\v\in\Delta} \|\enc[\D](\x+\v) - \enc[\D'](\x+\v) \|_2$. Notice that if the dictionary $\D$ has an encoder gap of at least $\tau_s^*$ for an input sample $\x$, then we can use \autoref{lemma:advers_D_stable} to obtain
    \[ \max_{\v\in\Delta} \|\enc[\D](\x+\v) - \enc[\D'](\x+\v) \|_2 \leq \frac{3}{2}  \frac{\sqrt{s}(1+\nu)^2}{\lambda(1-\eta_s)} \epsilon. \]
    Denote the probability of this event (that $\tau_s(\x)>\tau^*_s$) by $1-\rho$.  Note that $\epsilon\leq2\lambda/(1+\nu)^2$ is required in order to apply \autoref{lemma:advers_D_stable}, but this condition is mild and we will later show that this holds under the condition of minimal number of samples.  
    
    Likewise, $\tilde{\ell}$ is Lipschitz continuous w.r.t $\w$,
    \begin{align}
        \big| \tilde{\ell}(y,f_{\D,\w}(\x)) - \tilde{\ell}(y,f_{\D,\w'}(\x)) \big| \leq & L_\ell \max_{\v\in\Delta} | \langle \w^T, \enc[\D](\x+\v) \rangle  - \langle \w'^T, \enc[\D](\x+\v) \rangle  | \\
        \leq & \frac{ L_\ell(1+\nu)^2}{\lambda} \|\w-\w'\|_2,
    \end{align}    
    since $\max_{\v\in\Delta} \| \enc[\D](\x+\v)  \|_2 = (1+\nu)^2/\lambda$ ( \autoref{remark:Remarks1}).
    Furthermore, $\|\w-\w'\|_2\leq \epsilon$, as follows from our definition of $\epsilon$-cover.

    On the other hand, if $\D$ does not achieve this encoder gap ($\tau_s(\x)<\tau^*)$, which happens with probability $\rho$, then we can simply upper bound the worst possible loss, i.e. $ \big| \tilde{\ell}(y,f_{\D}(\x)) - \tilde{\ell}(y,f_{\D'}(\x)) \big| \leq b$. 
    
    Let us now analyze this probability, $\rho$. For simplicity, assume that $\tau_s(\x_i)$ are i.i.d. random variables. e.g. by computing $\tau_s(\x_i)$ on a held-out set with $m_2$ samples, independent from the $m_1$ samples that are used to train the dictionary. In particular, we split training and development samples $m_1$ and $m_2$ equally $m_1=m_2=m/2$. Let $F_{m_2}(\tau) := \frac{1}{m_2} \sum_{i=1}^{m_2} \textrm{1}_{\{\tau_s(\x_i) < \tau \}}$ denote the fraction of training points that achieve the encoder margin smaller than $\tau$. Let $F(\tau):=\text{Pr}(\tau_s(\x) < \tau)$. Then, uniform convergence~\citep{mohri2018foundations} yields that for any $\delta/2 > 0$, with probability at least $1-\delta/2$, we have that  $\sup_{\tau\in \mathbb{R}} |F_{m_2}(\tau) - F(\tau) | \leq  c \sqrt{\frac{\log(m_2) + \log(2/\delta)}{m_2}}$, for some constant $c$. Since this holds uniformly for any $\tau$, it holds in particular for $\tau = \tau^*_s$. This implies then that $F(\tau) = \text{Pr}(\tau_s(\x)\leq \tau^*) \leq c\sqrt{\frac{\log(m_2) + \log(1/\delta_2)}{m_2}} = \rho$.

    Note that the third term in Eq. \eqref{eq:3_parts} involves the expectation over the population, and so we can upper bound that term  by
    \[\frac{L_\ell(1+\nu)^2}{\lambda} \left( 1 +\frac{3}{2}  \frac{ B \sqrt{s}}{(1-\eta_s)} \right)\epsilon +  \frac{b}{\sqrt{m_2}} c\left(\sqrt{\log(m_2) + \log(2/\delta)}\right),  \]
    with probability at least $1-\delta/2$. 
    The second term, on the other hand, is the average loss over the training samples. For this, it suffices to note that the uniform bound $\sup_{\tau\in \mathbb{R}} |F_{m_2}(\tau) - F(\tau) |$ holds for \emph{any} choice of $\tau$. In particular, it holds for $\tau_s^*$ defined over both training and development samples. As a result, the dictionary satisfies the encoder gap on those samples, and so the second term can be simply upper bounded by 
    \[\frac{L_\ell(1+\nu)^2}{\lambda} \left( 1 +\frac{3}{2}  \frac{ B \sqrt{s}}{(1-\eta_s)} \right)\epsilon.\]

   We finally get expressions for the covering number as a function of $\epsilon$. For oblique manifolds (matrices of size $n\times p$ with unit norm columns), $N^{cov}(\mathcal{D},\epsilon) \leq (3/\epsilon)^{dp}$  \citep{seibert2019sample}, while for $B$-bounded vectors
   $N^{cov}(\mathcal{W},\epsilon) \leq (3B/\epsilon)^{p}$. Thus, the covering number of the direct product of the two constraint sets can be bounded by $N^{cov}(\mathcal{D},\mathcal{W},\epsilon) \leq (3/\epsilon)^{(d+1)p}~B^{p}$. 
   
   Gathering everything together, we can bound the generalization error by
    \begin{multline}
         \left| \tilde{R}_S(f)\! -\! \tilde{R}(f) \right| \leq b \sqrt{\frac{ (d+1)p \log(3/\epsilon) + p\log(B) + \log(4/\delta)}{m}} + bc\sqrt{2\frac{\log(m/2)+\log(2/\delta)}{m}} \\ +  \frac{2 L_\ell(1+\nu)^2}{\lambda} \left( 1 +\frac{3}{2}  \frac{ B \sqrt{s}}{(1-\eta_s)} \right)\epsilon.
    \end{multline} 
    
    All that remains is to set $\epsilon$ appropriately. Set $\epsilon = \lambda(1-\eta_s)/m$, and so
    \begin{multline}
          \left| \tilde{R}_S(f)\! -\! \tilde{R}(f) \right| \leq b \sqrt{\frac{ (d+1)p  \log(3m/(2\lambda(1-\eta_s))) + p\log(B) + \log{4/\delta}}{m}}  \\ + bc\sqrt{2\frac{\log(m/2)+\log(2/\delta)}{m}} +  12 \frac{ B \sqrt{s} L_\ell(1+\nu)^2}{m}.
    \end{multline}

    Lastly, the results above holds for $\epsilon$ small enough.
    Due to Lemma \autoref{lemma:preservation_of_sparsity}, one needs
     \[ \tau^*_s >  2 \nu +  \sqrt{\epsilon}\left( \sqrt{\frac{25}{\lambda}}(1+\nu) + 2 \left( \frac{(1+\nu)}{2\lambda} + 1\right) \right),\]
     implying that
     \[\sqrt{\frac{\lambda(1-\eta_s)}{m}} < \frac{\tau_s - 2\nu}{2 \left( 1+\frac{1+\nu}{2\lambda} \right) + (1+\nu) \sqrt{\frac{25}{\lambda}}}.\]
     Recalling that, naturally, $0<\sqrt{1/m}$, it is enough to require that $\tau_s>2\nu$ and that
    \begin{equation}\label{eq:condition_on_samples}
        m >  \frac{\lambda(1-\eta_s)}{(\tau_s - 2\nu)^2}  \left(  2 \left( 1+\frac{1+\nu}{2\lambda} \right) + (1+\nu) \sqrt{\frac{25}{\lambda}}  \right)^2.
    \end{equation}
    
    Lastly, we need to show that this condition guarantees the assumption $\epsilon\leq 2\lambda/(1+\nu)^2$ is satisfied, in order to apply  \autoref{lemma:advers_D_stable} above. Note that $\epsilon = \lambda(1-\eta_s)/m \leq \lambda/m$. Thus, we need $\lambda/m \leq 2\lambda/(1+\nu)^2$, which is satisfied as long as $m\geq  2 \geq(1+\nu)^2/2$, which is satisfied in all relevant scenarios.

\end{proof}

\subsection{Parameter adversarial stability} 
\label{supp:ParameterAdversarialStability}

In this section we prove the key result in  \autoref{lemma:advers_D_stable}, guaranteeing that the perturbation in the encoded features under model deviations and adversarial contamination is bounded. The main difficulty here is that the Lasso encoder solves a problem that is not strongly convex -- due to the overcompleteness of the dictionary -- and thus showing that the encoded features satisfy a Lipschitz property w.r.t the model parameters, particularly in the adversarial setting, is not trivial. 
As a result, this section will be dedicated to showing that if the model perturbation and adversarial contamination is small enough and there exist a positive encoder margin, then some sparsity is retained in the features after the respective perturbations. With this result at hand, the proof of \autoref{lemma:advers_D_stable} follows directly from the proof of Theorem 4 in \citep{mehta2013sparsity}, albeit with different constants which account for the perturbation $\v$. 
This \emph{preservation of sparsity} result is formalized later in  \autoref{lemma:preservation_of_sparsity}, and the following immediate lemmata will build some intermediate results needed for it. 

We first make a few remarks about the encoded features. We assume that $\x \in \mathcal{X} = \{\x:\|\x\|_2\leq1\}$, and recall that $\varphi_\D(\x) = \arg\min_\z \frac{1}{2} \|\x - \D\z\|^2_2 +\lambda \|\z\|_1$. We are interested in the result of the encoder when contaminated with an energy-bounded perturbation, namely
\begin{equation}\label{eq:adversarial_encoder_problem}
    \varphi_\D(\x_0+\v) = \arg\min_\z \frac{1}{2} \| (\x_0+\v) - \D\z\|^2_2 +\lambda \|\z\|_1,
\end{equation}
where $\v \in \Delta_\nu = \{\v : \|\v\|_2\leq \nu < 1 \}$. We will often denote $\x=\x_0+\v$ for simplicity. Also, note that there exist natural bounds for the penalty parameter, $0\leq\lambda\leq(1+\nu)$. The upper bound follows from the observation that as long as $\lambda>\|\D^T\x\|_{\infty}$, the solution of Eq.\eqref{eq:adversarial_encoder_problem} is the zero vector. Since the columns of $\D$ are normalized, $\|\D^T(\x_0+\v)\|_{\infty}\leq\|\x_0+\v\|_2\leq 1+\nu$.

Recall that from optimality conditions of Lasso, the solution $\enc[\D]{(\x)}$ satisfies
\begin{align}
    |\D_i^T(\x-\D\enc[\D]{(\x)})| =& ~\lambda \quad \text{if} \quad [\enc[\D]{(\x)}]_i \neq 0 \\
    |\D_i^T(\x-\D\enc[\D]{(\x)})| \leq & ~\lambda \quad \text{if} \quad [\enc[\D]{(\x)}]_i = 0.
\end{align}

Lastly, recall that the encoder gap assumption ($\tau_s\geq\tau*>0$) implies that there exist a set of inactive $(p-s)$ atoms $\mathcal{I}$ so that
\[ |\D_i^T(\x-\D\enc[\D]{(\x)})| < \lambda - \tau_s \]
for all $i \in \mathcal{I}$.

Let us now formalize a few properties on the solution of the Lasso solution that will be used throughout.

% --------------------------------------------------------------
\begin{remark}\label{remark:Remarks1} For the setting above, we have that 
\begin{enumerate}[a)] 
    \item $\|(\x_0+\v) - \D\enc[\D]{(\x_0+\v)}\|_2 \leq (1+\nu)$,
    
    \item $\|\enc[\D](\x_0+\v)\|_2 \leq (1+\nu)^2/(2\lambda)$,
    
    \item $\|\D \enc[\D](\x_0+\v)\|_2 \leq (1+\nu)$,
\end{enumerate}
\end{remark}
\begin{proof}
Remarks $a)$ and $b)$ can be shown by noting that, by definition of the encoder,
\begin{equation}
    \frac{1}{2} \| (\x_0+\v) - \D\enc[\D]{(\x_0+\v)}\|^2_2 +\lambda \|\enc[\D]{(\x_0+\v)}\|_1 \leq \frac{1}{2} \| (\x_0+\v) \|^2_2 \leq \frac{1}{2}(1+\nu)^2.
\end{equation}
The above follows since the LHS is the minimum function value, attained precisely $\enc[\D]{(\x_0+\v)}$, whereas $\frac{1}{2} \| (\x_0+\v) \|^2_2$ is the function value for the alternative choice of $\z=0$. The right-most inequality follows from the triangle inequality on $\|\x_0+\v\|_2$.

For remark $c)$, denote $\x=\x_0+\v$, and note that the minimizer of the above optimization problem satisfies (as follows from optimality of the minimizer \cite[Lemma 13 of Supplementary]{mehta2013sparsity}) % {\color{blue}{Raman: Question regarding the equation (5) that follows. What is the value of lambda at the minimizer? Also, Lemma 11 of [23] seems unrelated to the result you need here.}}
\begin{equation} \label{eq:pnt1}
    \frac{1}{2} \| \x - \D\enc[\D]{(\x)}\|^2_2 +\lambda \|\enc[\D]{(\x)}\|_1 = \frac{1}{2} \|\x\|^2_2 - \frac{1}{2}\|\D\enc[\D]{(\x)} \|^2_2.
\end{equation}
We expand the LHS and obtain a lower bound by Cauchy-Schwarz (and dropping the $\ell_1$ term)
\begin{align}
    \frac{1}{2} \| \x \| + \frac{1}{2}\|\D\enc[\D]{(\x)}\|^2_2 - \x^T\D\enc[\D]{(\x)} +\lambda \|\enc[\D]{(\x)}\|_1 \geq & ~  \frac{1}{2} \| \x \| + \frac{1}{2}\|\D\enc[\D]{(\x)}\|^2_2 - \|\x\|_2 \|\D\enc[\D]{(\x)}\|_2 \\
    \geq & ~ \frac{1}{2} \| \x \| + \frac{1}{2}\|\D\enc[\D]{(\x)}\|^2_2 - (1+\nu) \|\D\enc[\D]{(\x)}\|_2
\end{align}
Thus, together with \eqref{eq:pnt1}, we have that 
$\|\D \enc[\D](\x+\v)\|_2 \leq (1+\nu)$.
\end{proof}{}

\begin{lemma} \label{lemma:norm_stability_1}
If $\|\D-\D'\|\leq\epsilon\leq 2\lambda/(1+\nu)^2$, then
\begin{equation}
\max_{\v\in\Delta_\nu} \left| \| \D\enc[\D]{(\x_0+\v)} \|^2_2 - \|\D'\enc[\D']{(\x_0+\v)}\|^2_2 \right| \leq  \frac{5\epsilon}{2\lambda}(1+\nu)^2.
\end{equation}
\end{lemma}
The proof mimics that in \cite[Lemma 10-11]{mehta2013sparsity}, though accommodating for the adversarial perturbation. We include it here for completeness. Note that the above assumption on $\epsilon\leq 2\lambda/(1+\nu)^2$ is mild, and it will hold under the setting of later lemmata. 

\begin{proof}
Denote $\x = \x_0+\v$. Let us further denote the optimal value attained by the encoders with one and other model as
\[v^*_\D = \min_z \frac{1}{2} \|\x - \D\z\|^2_2 +\lambda \|\z\|_1,\]
\[v^*_{\D'} = \min_z \frac{1}{2} \|\x - \D'\z\|^2_2 +\lambda \|\z\|_1.\]
Then, since this cost is only increased if using a different representation, we have that:
\begin{align}
    v^*_{\D'} \leq & \frac{1}{2} \| \x - \D' \enc[\D]{(\x)} \|^2_2 + \lambda\|\enc[\D]{(\x)}\|_1 \\
    = & \frac{1}{2} \| \x - \D' \enc[\D]{(\x)} + (\D-\D)\enc[\D]{(\x)} \|^2_2 + \lambda\|\enc[\D]{(\x)}\|_1 \\
    \leq &  \frac{1}{2} \| \x - \D \enc[\D]{(\x)} \|^2_2 + \left| \langle \x - \D \enc[\D]{(\x)},(\D-\D')\enc[\D]{(\x)} \rangle \right| + \frac{1}{2} \|(\D-\D') \enc[\D]{(\x)}\|^2_2 + \dots \\
    & \qquad \dots + \lambda\|\enc[\D]{(\x)}\|_1  \\
    \leq &  v^*_{\D}  +  \| \x - \D \enc[\D]{(\x)} \|_2 \|\D-\D'\|_2 \|\enc[\D]{(\x)}\|_2 + \frac{1}{2} \left(\|\D-\D'\|_2 \|\enc[\D]{(\x)}\|_2\right)^2 \cmnt{by C.Swz.} \\
    \leq &  v^*_{\D}  + (1+\nu) \epsilon \frac{(1+\nu)^2}{2\lambda} + \frac{1}{2} \left( \frac{\epsilon(1+\nu)^2}{2\lambda} \right)^2 \cmnt{by \autoref{remark:Remarks1}}
    %\leq & v^*_{\D} + \frac{3\epsilon}{2\lambda}(1+\nu)^2,
\end{align}
We further simplify the expression above by noting that $\nu<1$ and that $\frac{\epsilon(1+\nu)^2}{2\lambda} \leq 1$ by assumption, obtaining
\begin{equation}
    v^*_{\D'} \leq v^*_{\D} + \frac{5\epsilon}{4\lambda}(1+\nu)^2.
\end{equation}
Thus, from this (and a symmetric argument) follows that 
\begin{equation}
   \left| v^*_{\D'} - v^*_\D \right| \leq \frac{5\epsilon}{4\lambda}(1+\nu)^2.
\end{equation}
Lastly, recall from Eq. \eqref{eq:pnt1} that $v^*_{\D} = \frac{1}{2}\|\x\|^2_2 - \|\D\enc[\D]{(\x)}\|^2_2$. Thus, using this expression for  $v^*_{\D}$ and  $v^*_{\D'}$ above, we get
\begin{equation}
    \left| \| \D\enc[\D]{(\x)} \|^2_2 - \|\D'\enc[\D']{(\x)}\|^2_2 \right| \leq 2 \left| v^*_{\D'} - v^*_\D \right| \leq  \frac{5\epsilon}{2\lambda}(1+\nu)^2.
\end{equation}

\end{proof}

% --------------------------------------------------------------

We now show that if the dictionaries are close, then the reconstructions from one and other encoded representation are not too far either. 

\begin{lemma}\label{lemma:bound_norms}
If $\|\D-\tD\|\leq\epsilon\leq 2 \lambda/(1+\nu)^2$, then
\begin{equation}
    \max_{\v\in\Delta_\nu} \|\D \enc[\D]{(\x_0+\v)} - \D \enc[\tD]{(\x_0+\v)}\|^2_2  \leq \frac{25\epsilon}{\lambda}(1+\nu)^2.
\end{equation}{}
\end{lemma}{}

\begin{proof}
For simplicity, denote $\x=\x_0+\v$, as well as $\alfa=\enc[\D]{(\x)}$ and $\talfa=\enc[\tD]{(\x)}$, where $\|\D-\tD\|_2\leq \epsilon$. We first upper bound $\left| \|\D\alfa \|^2_2 - \|\D \talfa\|^2_ 2 \right|$ by a sequence of algebraic manipulations:
\begin{align}
\left| \|\D\alfa \|^2_2 - \|\D \talfa\|^2_ 2 \right| \leq & \left|  \|\D\alfa \|^2_2  - \|\tD\talfa\|^2_ 2 \right| + \left|  \|\D\talfa\|^2_ 2  - \|\tD\talfa\|^2_ 2 \right| \cmnt{ $\pm \|\tD\talfa\|^2_2$, triang. inq.} \\
\cmnt{\autoref{lemma:norm_stability_1}, $\pm \tD$} \leq & \frac{5\epsilon}{2\lambda}(1+\nu)^2 + \left| \langle \D\talfa, (\D-\tD+\tD)\talfa \rangle -  \langle \tD\talfa, \tD\talfa \rangle \right|  \\
= & \frac{5\epsilon}{2\lambda}(1+\nu)^2 + \left| \langle \D\talfa, (\D-\tD)\talfa \rangle + \langle \D\talfa - \tD\talfa, \tD\talfa \rangle \right| \\
\cmnt{by $\pm \D$} ~~ = & \frac{5\epsilon}{2\lambda}(1+\nu)^2 + \left| \langle \D\talfa, (\D-\tD)\talfa \rangle + \langle \D\talfa - \tD\talfa, (\tD-\D+\D)\talfa \rangle \right| \\
% \leq & \frac{3\epsilon}{\lambda}(1+\nu)^2 + \left| \langle \D\talfa, (\D-\tD)\talfa \rangle + \langle (\D-\tD)\talfa , (\tD-\D+\D)\talfa \rangle \right| \\
= \frac{5\epsilon}{2\lambda}(1+\nu)^2 +& \left| \langle \D\talfa, (\D-\tD)\talfa \rangle + \langle (\D-\tD)\talfa , \D\talfa \rangle - \langle (\D-\tD)\talfa , (\D-\tD)\talfa \rangle \right|  \\
\leq & \frac{5\epsilon}{2\lambda}(1+\nu)^2 + 2 \left| \langle \D\talfa, (\D-\tD)\talfa \rangle  \right| \cmnt{by dropping $-\|(\D-\tD)\talfa\|^2_2$}\\
\leq & \frac{5\epsilon}{2\lambda}(1+\nu)^2 + 2 \| \D\talfa \|_2 \|\D-\tD\|_2 \|\talfa\|_2 \cmnt{by C.S. and operator norm} \\
\leq & \frac{5\epsilon}{2\lambda}(1+\nu)^2 + 2 \frac{\epsilon}{2\lambda}(1+\nu)^2 \| \D\talfa \|_2 \cmnt{by \autoref{remark:Remarks1}}
\end{align}
The term $\|\D\talfa\|_2$ cannot be directly bounded via \autoref{remark:Remarks1} because $\talfa$ is the representation computed via $\tD$ (not $\D$). Instead, by letting $\Delta = \D-\tD$, we can simplify the above bound as $ \| \D\talfa \|_2 \leq \| \tD\talfa\|_2 +\|\Delta\talfa \|_2 \leq (1+\nu) + \epsilon (1+\nu)^2/(2\lambda) \leq 2 + 1 = 3$. Then, resuming above,
\begin{align} \label{eq:stability_norms_recs}
\left| \|\D\alfa \|^2_2 - \|\D \talfa\|^2_ 2 \right| \leq & \frac{5\epsilon}{2\lambda}(1+\nu)^2 + \frac{6\epsilon}{2\lambda}(1+\nu)^2   
= \frac{11\epsilon}{2\lambda}(1+\nu)^2.
\end{align}
Now, by definition of $\alfa$ (as the minimizer of the Lasso problem), we have that
\begin{align}\label{eq:upper_bound_vD}
    v^*_{\D}(\x) = \frac{1}{2} \| \x -\D \alfa \|^2_2 + \lambda \|\alfa\|_1 \leq \frac{1}{2} \left\| \x -\D \left(\frac{\alfa+\talfa}{2}\right) \right\|^2_2 + \lambda \left\|\frac{\alfa+\talfa}{2} \right\|_1.
\end{align}
We now expand the RHS above through the same algebraic manipulations:
\begin{align}
    v^*_{\D}(\x) & \leq \frac{1}{2} \left\| \x -\D \left(\frac{\alfa+\talfa}{2}\right) \right\|^2_2 + \lambda \left\|\frac{\alfa+\talfa}{2} \right\|_1 \\
    & =  \frac{1}{2} \left( \|\x \|^2_2 - \langle \x ,  (\D \alfa+ \D \talfa) \rangle + \frac{1}{4}\| \D \alfa+ \D \talfa \|^2_2\right) + \lambda \left\|\frac{\alfa+\talfa}{2} \right\|_1 \\
    & = \frac{1}{2} \|\x \|^2_2 - \frac{1}{2} \langle \x ,  \D \alfa \rangle - \frac{1}{2} \langle \x ,  \D \talfa \rangle + \frac{1}{8} \left( \| \D \alfa\|_2^2 + \|\D \talfa \|^2_2 + 2\langle \D\alfa , \D\talfa \rangle \right) + \lambda \left\|\frac{\alfa+\talfa}{2} \right\|_1 \\ \label{eq:long_1}
    & \leq \frac{1}{2} \|\x \|^2_2 - \frac{1}{2} \langle \x ,  \D \alfa \rangle - \frac{1}{2} \langle \x ,  \D \talfa \rangle + \frac{1}{4} \| \D \alfa\|_2^2 +  \frac{1}{4} \langle \D\alfa , \D\talfa \rangle + \dots \\ & \qquad \dots + \frac{\lambda}{2} \left\| \alfa \right\|_1 + \frac{\lambda}{2} \left\|\talfa\right\|_1 + \frac{11}{16}\frac{\epsilon}{\lambda}(1+\nu)^2,
\end{align}
where the last step follows by adding and subtracting $\|\D\alfa\|^2_2$ and employing the bound obtained above in \eqref{eq:stability_norms_recs}. Now, from Eq. \eqref{eq:pnt1}, we can write
\begin{equation}\label{eq:expression_for_L1_term}
\lambda \|\alfa\|_1 = \langle \x-\D\alfa , \D\alfa \rangle = \langle \x , \D\alfa\rangle - \|\D\alfa\|^2_2.
\end{equation}
The expression for $\|\talfa\|_1$ is expanded similarly but then upper bounded via  \autoref{lemma:norm_stability_1} by adding and subtracting $\|\D\alfa\|^2_2$:
\begin{align}
    \lambda \|\talfa\|_1 = \langle \x-\tD\talfa , \tD\talfa \rangle &\leq \langle \x , \tD\talfa \rangle - \|\D\alfa\|^2_2 + \frac{5\epsilon}{2\lambda}(1+\nu)^2 \\
    &= \langle \x , \D\talfa \rangle + \langle  \x , (\tD-\D)\talfa \rangle - \|\D\alfa\|^2_2 + \frac{5\epsilon}{2\lambda}(1+\nu)^2 \cmnt{by $\pm \D$} \\
   \cmnt{by C.S.}  &\leq  \langle \x , \D\talfa \rangle + \|\x\|_2 \|\tD-\D\|_2\|\talfa\|_2 - \|\D\alfa\|^2_2 + \frac{5\epsilon}{2\lambda}(1+\nu)^2 \\
    &\leq  \langle \x , \D\talfa \rangle + (1+\nu)\epsilon \frac{(1+\nu)^2}{2\lambda} - \|\D\alfa\|^2_2 + \frac{5\epsilon}{2\lambda}(1+\nu)^2 \\ \label{eq:expression_for_ell1_tilde}
    &\leq  \langle \x , \D\talfa \rangle - \|\D\alfa\|^2_2 + \frac{7\epsilon}{2\lambda}(1+\nu)^2.
\end{align}
Thus, we can now upper bound the expression for $v^*_\D$ in Eq. \eqref{eq:pnt1} by combining Eq. \eqref{eq:long_1}, \eqref{eq:expression_for_L1_term} and \eqref{eq:expression_for_ell1_tilde} as follows. 
From Eq. \eqref{eq:pnt1} and the bound in Eq. \eqref{eq:long_1} we get:
\begin{align}
  v^*_\D =&  \frac{1}{2} \| \x \|^2_2 - \frac{1}{2} \|\D\alfa\|^2_2 \\
  \leq& \frac{1}{2} \|\x \|^2_2 - \frac{1}{2} \langle \x ,  \D \alfa \rangle - \frac{1}{2} \langle \x ,  \D \talfa \rangle + \frac{1}{4} \| \D \alfa\|_2^2 +  \frac{1}{4} \langle \D\alfa , \D\talfa \rangle + \dots... \\ &\qquad \dots + \frac{\lambda}{2} \left\| \alfa \right\|_1 + \frac{\lambda}{2} \left\|\talfa\right\|_1 + \frac{11}{16}\frac{\epsilon}{\lambda}(1+\nu)^2.
\end{align}
Replacing now the expression for $\lambda\|\alfa\|_1$ from \eqref{eq:expression_for_L1_term} and the upper bound for $\lambda\|\talfa\|_1$ from \eqref{eq:expression_for_ell1_tilde}:
\begin{align}
  v^*_\D =  \frac{1}{2} \| \x \|^2_2 - \frac{1}{2} \|\D\alfa\|^2_2 \leq \frac{1}{2} \|\x \|^2_2 - \frac{3}{4} \| \D \alfa\|_2^2 + \frac{1}{4} \langle \D\alfa , \D\talfa \rangle + \frac{39}{16} \frac{\epsilon}{\lambda}(1+\nu)^2,
\end{align}
which leads to
\begin{align}
- \frac{1}{2} \|\D\alfa\|^2_2 \leq - \frac{3}{4} \| \D \alfa\|_2^2 + \frac{1}{4} \langle \D\alfa , \D\talfa \rangle + \frac{39}{16} \frac{\epsilon}{\lambda}(1+\nu)^2,
\end{align}
and so
\begin{equation}\label{eq:bound_on_norm_1}
    \|\D \alfa \|^2_2 \leq \langle \D\alfa , \D\talfa \rangle +  \frac{39}{4}\frac{\epsilon}{\lambda} (1+\nu)^2.
\end{equation}
Finally, with this expression we can now bound the distance 
\begin{align}
    \| \D\alfa - \D\talfa \|^2_2 = &~ \| \D\alfa \|^2_2 + \| \D\talfa \|^2_2 - 2 \langle \D\alfa , \D\talfa \rangle \\
    \leq & ~ \| \D\alfa \|^2_2 + \| \D\talfa \|^2_2 - 2 \| \D\alfa \|^2_2 + \frac{39}{2} \frac{\epsilon}{\lambda} (1+\nu)^2 \cmnt{by the expression above} \\
    = & ~  \| \D\talfa \|^2_2 -  \| \D\alfa \|^2_2 + \frac{39}{2} \frac{\epsilon}{\lambda} (1+\nu)^2  \\
    \leq &~ 25 \frac{  \epsilon}{\lambda} (1+\nu)^2,
\end{align}
where the last inequality follows from Eq. \eqref{eq:stability_norms_recs}. 

\end{proof}

% --------------------------------------------------------------

We will now show that if the dictionaries are close enough and if the solution of one of them was at most $s$-sparse and has a positive encoder gap, then the solution with the perturbed model retains the sparsity. This is inspired by the result in \citep{mehta2013sparsity}. However, because of the adversarial perturbation, extra work is required to provide a condition on the \emph{unperturbed} gap, i.e. that which will withstand adversarial energy-bounded perturbation. The following Lemma will be necessary to show the result:

\begin{lemma} \label{lemma:adversarial_norm_stability}
If $\|\v\|_2 \leq \nu$, then
\begin{equation}
    \| \D \enc[\D]{(\x)} - \D \enc[\D]{(\x+\v)} \|^2_2 \leq \nu^2.
\end{equation}
\end{lemma}
Let us postpone the proof of this result for later. We are now ready to state and prove the preservation of sparsity result:

\begin{lemma}{(Preservation of sparsity under model deviation and adversarial perturbations)} \label{lemma:preservation_of_sparsity}
Consider $\enc[\D]{(\x_0+\v)}$, for $\|\v\|_2\leq\nu$, and an alternative dictionary $\tD$ so that $\|\D - \tD\|_2\leq \epsilon\leq2\lambda/(1+\nu)^2$. If there exist a set of inactive $(p-s)$ atoms $\mathcal{I}$ so that
\begin{equation}\label{eq:enc_gap_ineq2}
     |\D_i^T(\x_0 -\D\enc[\D]{(\x_0)})| < \lambda - \tau_s
\end{equation}
for all $i \in \mathcal{I}$, and 
\begin{equation}
    \tau_s >  2 \nu +  \sqrt{\epsilon}\left( \sqrt{\frac{25}{\lambda}}(1+\nu) + 2 \left( \frac{(1+\nu)}{\lambda} + 1\right) \right),
\end{equation}
then $[\enc[\tD]{(\x_0+\v)}]_i = 0~\forall i\in\mathcal{I}$, where (reminder)
\begin{equation}
     \enc[\tD]{(\x_0+\v)} = \arg\min_\z \frac{1}{2} \| (\x_0+\v) - \tD\z\|^2_2 +\lambda \|\z\|_1.
\end{equation}

\end{lemma}

\begin{proof}
Let $\x = \x_0 + \v$, as well as $\alfa = \enc[\D]{(\x)}$, $\talfa = \enc[\tD]{(\x)}$, and let $\mathcal{I}$ be the set of $(p-s)$ inactive atoms with positive gap $\tau_s$.

In order for the inactive set of atoms $\mathcal{I}$ to remain inactive, we need to show that $\forall i \in \mathcal{I}$, \[\left| \langle  \tD_i, \x - \tD \talfa \rangle \right| < \lambda. \]
Consider the following upper bound to the LHS above:
\begin{align}
    \left| \langle  \tD_i, \x - \tD \talfa \rangle \right| &\leq \left| \langle  \D_i, \x - \tD  \talfa \rangle \right| + \|\tD_i - \D_i\|_2 \|\x-\tD \talfa\|_2 \cmnt{by $\pm \D_i$ and C.S.}\\
    &\leq \left| \langle  \D_i, \x - \tD  \talfa \rangle \right| + \epsilon (1+\nu) \\
    &\leq \left| \langle  \D_i, \x - \D  \talfa \rangle \right| + \left| \langle  \D_i, (\tD - \D) \talfa \rangle \right| + \epsilon (1+\nu) \cmnt{by $\pm \D$ and triang ineq.} \\
    &\leq \left| \langle  \D_i, \x - \D  \talfa \rangle \right| + \|\D_i\|_2 \|\tD - \D\|_2 \| \talfa \|_2 + \epsilon (1+\nu) \\
    &\leq \left| \langle  \D_i, \x - \D  \talfa \rangle \right| + \frac{\epsilon}{2\lambda} (1+\nu)^2 + \epsilon (1+\nu). \cmnt{by  \autoref{remark:Remarks1} and unit-norm columns}
\end{align}

Thus, it is sufficient to show that
\begin{equation}\label{eq:requirement_1}
    \left| \langle  \D_i, \x - \D  \talfa \rangle \right| < \lambda - \epsilon(1+\nu) \left( \frac{(1+\nu)}{2\lambda} + 1\right).
\end{equation}
Let us now replace $\x$, $\alfa$ and $\talfa$ by their definitions and upper bound the left hand side above by using  \autoref{lemma:bound_norms} and  \autoref{lemma:adversarial_norm_stability}
\begin{align}
    \left| \langle  \D_i, \x - \D  \talfa \rangle \right| =& \left| \langle  \D_i, (\x_0+\v) - \D \enc[\tD]{(\x_0+\v)} \rangle \right| \\
    \leq& \big| \langle \D_i,(\x_0+\v)-\D\enc[\D]{(\x_0+\v)} \big| + \left| \langle \D_i, \D\enc[\D]{(\x_0+\v)} - \D\enc[\tD]{(\x_0+\v)} \right| \cmnt{by $\pm \alfa$}\\
    \leq& \overbrace{\left| \langle \D_i,\x_0-\D\enc[\D]{(\x_0+\v)} \right| + \|\D_i\|_2 \|\v\|_2}{} + \|\D_i\|_2 \| \D\enc[\D]{(\x_0+\v)} - \D\enc[\tD]{(\x_0+\v)} \|_2  \\
     \leq& \left| \langle \D_i,\x_0-\D\enc[\D]{(\x_0+\v)} \right| + \nu +  \sqrt{\frac{25\epsilon}{\lambda}}(1+\nu)  \cmnt{by  \autoref{lemma:bound_norms}}\\
     \cmnt{by $\pm \enc[\D]{(\x_0)}$} \quad \leq& \left| \langle \D_i,\x_0-\D\enc[\D]{(\x_0)} \right| + \|\D_i\|_2\| \D\enc[\D]{(\x_0)} - \D\enc[\D]{(\x_0+\v)} \|_2 + \nu +  \sqrt{\frac{25\epsilon}{\lambda}}(1+\nu)   \\
     \leq& \left| \langle \D_i,\x_0-\D\enc[\D]{(\x_0)} \right| + 2 \nu + \sqrt{\frac{25\epsilon}{\lambda}}(1+\nu)  \cmnt{by \autoref{lemma:adversarial_norm_stability}} \\
    < & \lambda - \tau_s + 2 \nu +  \sqrt{\frac{25\epsilon}{\lambda}}(1+\nu)
\end{align}
where the last step follows from the assumption of the encoder gap in Eq. \eqref{eq:enc_gap_ineq2}. Thus, merging with \eqref{eq:requirement_1}, we require
\[ - \tau_s + 2 \nu + \sqrt{\frac{25\epsilon}{\lambda}}(1+\nu) < - \epsilon(1+\nu) \left( \frac{(1+\nu)}{2\lambda} + 1\right), \]
implying that as long as
\begin{equation}
    \tau_s >  2 \nu  +  \sqrt{\frac{25\epsilon}{\lambda}}(1+\nu) + \epsilon(1+\nu) \left( \frac{(1+\nu)}{2\lambda} + 1\right) 
\end{equation}
the inactive set $\mathcal{I}$ remains inactive.
For the sake of simplicity, we will make the above condition more stringent. Using the fact that $\nu<1$ and that $\epsilon \leq 1$. Thus,
\begin{equation}
    \tau_s >  2 \nu +  \sqrt{\epsilon}\left( \sqrt{\frac{25}{\lambda}}(1+\nu) + 2 \left( \frac{(1+\nu)}{2\lambda} + 1\right) \right).
\end{equation}
\end{proof}

% --------------------------------------------------------------

The lemma above is central, as it guarantees that a sparsity of up to $s$ non-zeros is retained under model deviations and adversarial perturbations.  \autoref{lemma:norm_stability_1} now follows directly from the proof of Theorem 4 in \citep{mehta2013sparsity}, albeit with the constants provided by  \autoref{remark:Remarks1} which account for the perturbation $\v$.

% -------------------------------------------------------------
We owe the proof of \autoref{lemma:adversarial_norm_stability}. This Lemma will also be instrumental in the proof of \autoref{thm:certificate}. We now re-state it, and proceed to prove it.

\begin{lemma}{5.4} \label{lemma:norm_stability_adverarial} (Norm Stability under adversarial perturbations)\\
If $\|\v\|_2 \leq \nu$, then
\begin{equation}
    \| \D \enc[\D]{(\x)} - \D \enc[\D]{(\x+\v)} \|^2_2 \leq \nu^2
\end{equation}{}
\end{lemma}

\begin{proof}
We will re-formulate the Lasso problem as an equivalent quadratic program, and then utilize optimality properties of its solution. 
Let us define the vector $\bz \in \mathbb{R}^{3p}$ such that $\bz = [\z,\z^+,\z^-]^T$, with $\z^+$ and $\z^-$ containing all positive and negative elements in $\z$, respectively. Define then the following quadratic cost
\begin{equation}
    Q(\bz,\x) = \frac{1}{2} \bz^T  \begin{bmatrix}
   \D^T\D & \mathbf{0}_{p\times 2p}  \\
   \mathbf{0}_{2p\times p}  & \mathbf{0}_{2p\times 2p}
   \end{bmatrix}
   \bz - \bz^T \begin{bmatrix}
   \D^T \\ \mathbf{0}_{2p\times d}
   \end{bmatrix}
   \x + \lambda [\mathbf{0}_{p}^T, \mathbf{1}_{2p}^T] \bz.
   \end{equation}
With this definition, the Lasso problem can be re-formulated as the following quadratic program:
\begin{equation}\label{eq:quad_program}
    \min_{\bz\in\mathbb{R}^{3p}} ~Q(\bz,\x) \quad \text{subject to} \quad \bz\in\mathcal{K} = \{\bz : \z = \z^+ - \z^- ; ~ \z^+,\z^- \geq 0\}.
\end{equation} 
Let us denote $Q(\bz) = Q(\bz,\x_0)$ and $\tilde{Q}(\bz) = Q(\bz,\x_0+\v)$ for short, and $\betta$ and $\tilde{\betta}$ as the solution to the quadratic program with $Q(\bz)$ and $\tilde{Q}(\bz)$, respectively. Moreover, denote $\alfa = \enc[\D]{(\x_0)}$ and $\talfa = \enc[\D]{(\x_0+\v)}$. 
With this notation, note that \[\betta = \begin{bmatrix} \alfa \\ \alfa^+ \\ \alfa^-
\end{bmatrix} \quad\text{and}\quad \tbetta = \begin{bmatrix} \talfa \\ \talfa^+ \\ \talfa^-
\end{bmatrix}.
\]

Note that the above problem in \eqref{eq:quad_program} is the minimization of a convex differentiable function over a convex set and therefore, for every $\bz\in\mathcal{K}$,
\begin{subequations}
\begin{align} \label{eq:quad_prog_normal}
    (\bz - \betta)^T \nabla_{\z} Q(\betta) \geq& 0 \\ \label{eq:quad_prog_advs}
    (\bz - \tilde{\betta})^T \nabla_{\z} \tilde{Q}(\tilde{\betta}) \geq& 0.
\end{align}
\end{subequations}
This gradient can be written as 
\begin{equation}
\nabla Q(\bz) =  \begin{bmatrix}
   \D^T\D & \mathbf{0}_{p\times 2p}  \\
   \mathbf{0}_{2p\times p}  & \mathbf{0}_{2p\times 2p}
   \end{bmatrix}
   \bz - \begin{bmatrix}
   \D^T \\ \mathbf{0}_{2p\times d}
   \end{bmatrix}
   \x + \lambda \begin{bmatrix} \mathbf{0}_{p} \\ \mathbf{1}_{2p}
   \end{bmatrix}.
\end{equation}
Now, choosing $\tilde{\betta}$ as $\bz$ in \eqref{eq:quad_prog_normal}, ${\betta}$ as $\bz$  in \eqref{eq:quad_prog_advs} and subtracting one from the other, we get
\begin{align}
(\tbetta-\betta)^T\left( \nabla Q(\betta) -  \nabla \tilde{Q}(\tbetta)  \right) \geq 0 
\end{align}
which after employing the definitions for $\betta,\tbetta,\nabla Q$ and $\nabla\tilde{Q}$ results in 
\begin{align}
    (\talfa - \alfa)^T \big( \D^T\D (\alfa - \talfa) + \D^T\v \big)\geq& 0.
\end{align}
The lemma is proven by finally expanding the above and employing Cauchy-Schwarz:
\begin{align}
    \|\D\talfa - \D\alfa\|^2_2 \leq& (\talfa - \alfa)^T \D^T\v \leq \|\v\|_2 \|\D\talfa - \D\alfa\|_2 \leq \nu \|\D\talfa - \D\alfa\|_2.
\end{align}
\end{proof}

\section{Robustness Certificate}
\label{supp:certificate}

\begin{theorem} [Robustness certificate for binary predictive sparse coding]  \label{thm:certificate}
Consider the predictor $f_{\D,\w}(\x)$, computed via $\enc[\D]{(\x)}$ with an encoder gap of $\tau_s(\x)$ and $\eta_s$-RIP dictionary $\D$. Then,
\begin{equation}
    \text{sign}(f_{\D,\w}(\x)) = \text{sign}(f_{\D,\w}(\x+\v)),\quad \forall \v:\|\v\|_2 \leq \nu
\end{equation}
so long as $\nu < \min\{~ \tau_s(\x)/2~ ,~ \rho_\x \sqrt{1-\eta_s}~ \}.$
\end{theorem}

We now proceed to prove \autoref{thm:certificate}. We first must show that if there exist a positive encoder gap for a particular inactive set, this set will remain inactive under adversarial perturbations. This follows as a particular case of \autoref{lemma:preservation_of_sparsity} with $\epsilon = 0$, i.e. when there is no difference between the dictionaries: $\|\D-\tD\|_2 = 0$. We re-state it here for completeness in this simplified form.

\begin{corollary} \label{corollary:preservation_of_sparsity}
Consider $\enc[\D]{(\x_0)}$ and $\enc[\D]{(\x_0+\v)}$, for $\|\v\|_2\leq\nu$. If there exist a set of inactive $(p-s)$ atoms $\mathcal{I}$ in $\enc[\D]{(\x_0)}$ so that
\begin{equation}\label{eq:enc_gap_ineq}
     |\D_i^T(\x_0 -\D\enc[\D]{(\x_0)})| < \lambda - \tau_s
\end{equation}
for all $i \in \mathcal{I}$, and 
\begin{equation}
    \tau_s >  2 \nu,
\end{equation}
then $[\enc[\tD]{(\x_0+\v)}]_i = 0~\forall i\in\mathcal{I}$.
\end{corollary}

With this result, we now present a Lemma guaranteeing that the original and adversarially perturbed representation are not too far.

\setcounter{section}{5}
\renewcommand{\thesection}{\arabic{section}}
\setcounter{theorem}{1}

\begin{lemma}[Stability of representations under adversarial perturbations]
Let $\enc[\D]{(\x_0)}$ and $\enc[\D]{(\x_0+\v)}$, for $\|\v\|_2\leq\nu$. If $\enc[\D]{(\x_0)}$ has an encoder gap $\tau_s >  2 \nu$, and the dictionary is RIP with constant $\eta_s$, then
\begin{equation}
    \| \enc[\D]{(\x_0)} - \enc[\D]{(\x_0+\v)} \|_2 \leq \frac{\nu}{\sqrt{1-\eta_s}}.
\end{equation}
\end{lemma}

\setcounter{section}{2}
\renewcommand{\thesection}{\Alph{section}}

\begin{proof}
The proof of this result is now simple given our previous developments. On one hand, we have from  \autoref{lemma:norm_stability_adverarial} that
\begin{equation}
    \|\D\enc[\D]{(\x_0)} - \D\enc[\D]{(\x_0+\v)}\|^2_2\leq \nu^2.
\end{equation}
On the other hand, since $\enc[\D]{(\x_0)}$ has an encoder gap of $\tau_s >  2 \nu$, there exist an inactive set of $(p-s)$ atoms that is retained in $\enc[\D]{(\x_0+\v)}$ by  \autoref{corollary:preservation_of_sparsity}. Thus, $\|\enc[\D]{(\x_0)} - \enc[\D]{(\x_0+\v)}\|_0\leq s$. As a result, since $\D$ is $\eta_s$-RIP, we can write
\begin{equation}
    \|\D\enc[\D]{(\x_0)} - \D\enc[\D]{(\x_0+\v)}\|^2_2 \geq (1-\eta_s)\|\enc[\D]{(\x_0)} - \enc[\D]{(\x_0+\v)}\|^2_2.
\end{equation}
Combining the lower and upper bounds proves the claim.
\end{proof}

% --------------------------------------------------------

We are now ready to prove the result in \autoref{thm:certificate}. 

\begin{proof}
The proof is simple and inspired by the analysis in \citep{romano2019adversarial}.
Recall that the hypothesis is implemented by $f_{\D,\w}(\x) = \langle \w , \enc[\D]{(\x)} \rangle$. Since $\enc[\D]{(\x)}$ has an encoder gap $ \tau_s \geq 2\nu$, then it follows from the above  \autoref{lemma:stability_of_rep_adversarial} that
\begin{equation}
    \| \enc[\D]{(\x_0)} - \enc[\D]{(\x_0+\v)} \|_2 \leq \frac{\nu}{\sqrt{1-\eta_s}}.
\end{equation}
Without loss of generality, consider the case when $f_{\D,\w}(\x)>0$. Let us lower bound $f_{\D,\w}(\x+\v)$ as follows:
\begin{align}
    \langle \w , \enc[\D]{(\x+\v)} \rangle =& ~\langle \w , \enc[\D]{(\x)} \rangle + \langle \w , \enc[\D]{(\x+\v)} - \enc[\D]{(\x)} \rangle \\
    \geq& ~\rho_\x\|\w\|_2 - | \langle \w , \enc[\D]{(\x+\v)} - \enc[\D]{(\x)} \rangle | \\
    \geq& ~\|\w\|_2\left(\rho_\x - \frac{\nu}{\sqrt{1-\eta_s}}\right).  \\
\end{align}
Therefore, as long as $\rho_x > \frac{\nu}{\sqrt{1-\eta_s}}$ (and $\w\neq\mathbf{0}$), $sign(f_{\D,\w}(\x)) = sign(f_{\D,\w}(\x+\v))$.
\end{proof}

\setcounter{section}{5}
\renewcommand{\thesection}{\arabic{section}}
\setcounter{theorem}{0}

\begin{theorem}[Robustness Certificate for multiclass supervised sparse coding] Let $\rho_x$ be the multiclass classifier margin of $f_{\D,\w}(\x)$, with $\rho_\x>0$, composed of an encoder with gap of $\tau_s(\x)$ and $\eta_s$-RIP dictionary $\D$. Furthermore, denote by $c_\W = \max_{i\neq j}\|\W_i-\W_j\|_2$ Then,
    \begin{equation}
        \arg\max_{j\in[K]} ~[\W^T f_{\D,\w}(\x)]\ = \arg\max_{j\in[K]}~ [\W^T f_{\D,\w}(\x+\v)],\quad \forall~ \v:\|\v\|_2\leq \nu 
    \end{equation}
    so long as $\nu \leq \min\{ \tau_s(\x)/2 , \rho_\x \sqrt{1-\eta_s} / c_\W \}.$
\end{theorem} 

\setcounter{section}{3}
\renewcommand{\thesection}{\Alph{section}}

\begin{proof}
Consider a sample with a positive multiclass margin:
\[\rho_\x = \W^T_{y} \enc[\D](\x) - \max_{j\neq y} \W^T_{j} \enc[\D](\x) > 0, \]
and let us lower-bound the margin on the perturbed input $f_{\D,\w}(\x+\v)$ as follows:
\begin{align}
    \rho_{\x+\v} =& \W^T_{y} \enc[\D](\x+\v) - \max_{j\neq y} \W^T_{j} \enc[\D](\x+\v) \\
    =& \min_{j\neq y} \langle \W_y - \W_j , \enc[\D]{(\x+\v)} \rangle \\
    \geq& ~ \min_{j\neq y} \langle \W_y - \W_j , \enc[\D]{(\x)} \rangle - \| \W_y - \W_j\|_2 \|\enc[\D]{(\x+\v)} - \enc[\D]{(\x)}\|_2 \\
    \geq& ~ \rho_\x - c_\W~ \nu / \sqrt{1-\eta_s},
\end{align}
where the second-to-last inequality follows from hypothesis and \autoref{lemma:stability_of_rep_adversarial}.
Therefore, as long as $\rho_x > \frac{c_\W \nu}{\sqrt{1-\eta_s}}$, $\rho_{\x+\v}>0$.

\end{proof}

\section{Numerical Experiments Details}
\label{supp:numerical}

The models on images (MNIST and CIFAR10) were trained by minimizing the following regularized empirical risk
\begin{equation}\label{eq:training_supp}
    \min_{\W,\D}~ \frac{1}{m} \sum_{i=1}^m \ell\left(y_i,\langle \W,\enc[\D]{(\x_i)} \rangle\right) + \alpha \|\I - \D^T\D\|^2_F + \beta \|\W\|^2_F,
\end{equation}
over the training set with $m$ samples. We use the default training/testing split provided in the datasets. The difficulty in this optimization problem resides in computing the derivative of this loss w.r.t. the dictionary $\D$ via the solution of the encoder $\enc[\D]{(\x)}$. Our approach relies on using an approximate but differentiable solution for $\enc[\D]{(\x)}$: we compute the features (by solving the corresponding Lasso problem) via Fast Iterative Soft Thresholding \citep{beck2009fast}. This algorithm enjoys a fast convergence rate of $\mathcal O(1/T^2)$, and we use $T=25$ iterations within the optimization problem above.

We found it useful to pre-train the model, $\D$, in an unsupervised manner first. This is done by simply minimizing a regression problem of the form
\[\min_{\D}~ \frac{1}{m} \sum_{i=1}^m \|\x_i -\D \enc[\D]{(\x_i)} \|^2_2.\]

Additionally, when performing the supervised learning stage, if progressively increase the value of $\lambda$ through the iterations until the pre-specified target value (which where set to 0.2 and 0.3 in \autoref{fig:mnist_comparison_a} and \autoref{fig:mnist_comparison_a}, respectively). We employ Adam \citep{kingma2014adam} with a mini-batch size of 128, and train for 35 epochs. The dictionary is normalized after each weight-update. All other hyper-parameters are detailed in the accompanying code.

At deployment time, however, it is important that the solution computed by $\enc[\D]{(\x)}$ is exact, because the encoder gap $\tau_s$ is defined in terms of these optimality conditions. Therefore, we use FISTA to find the estimated support of the solution, and then compute the exact solution analytically given this support.

All experiments were coded in Python and employing pytorch for GPU acceleration. All other employed packages are detailed in the accompanying code.

\end{document}